\documentclass[10pt,twocolumn,a4paper]{article}

\usepackage[a4paper,top=17mm,bottom=19mm,left=17mm,right=17mm,columnsep=6mm]{geometry}
\usepackage[T1]{fontenc}
\usepackage{amsmath,amsthm}
\usepackage{newtxtext,newtxmath}
\usepackage{microtype}
\usepackage{graphicx}
\usepackage{bm,mathtools}
\usepackage[ruled,vlined,linesnumbered]{algorithm2e}
\usepackage{booktabs}
\usepackage[numbers,sort&compress]{natbib}
\usepackage{xcolor}
\definecolor{linkblue}{RGB}{20,70,135}
\usepackage[colorlinks=true,linkcolor=linkblue,citecolor=linkblue,urlcolor=linkblue]{hyperref}
\usepackage{xr-hyper}
\usepackage{url}
\usepackage{enumitem}
\usepackage{titlesec}
\usepackage{caption}
\usepackage{subcaption}

\titleformat{\section}{\large\bfseries}{\thesection}{0.5em}{}
\titlespacing*{\section}{0pt}{5pt}{4pt}
\titleformat{\subsection}{\normalsize\bfseries}{\thesubsection}{0.5em}{}
\titlespacing*{\subsection}{0pt}{4pt}{2pt}
\newtheoremstyle{bmfaplain}{3pt}{3pt}{\normalfont}{}{\bfseries}{.}{0.5em}{}
\newtheoremstyle{bmfadefinition}{3pt}{3pt}{\normalfont}{}{\bfseries}{.}{0.5em}{}
\theoremstyle{bmfaplain}
\newtheorem{theorem}{Theorem}[section]
\newtheorem{lemma}[theorem]{Lemma}
\newtheorem{proposition}[theorem]{Proposition}
\newtheorem{corollary}[theorem]{Corollary}
\theoremstyle{bmfadefinition}

\SetAlgorithmName{Algorithm}{Algorithm}{List of Algorithms}
\SetKwInput{KwIn}{Input}
\SetKwInput{KwOut}{Output}

\title{\textbf{BMFA: Boundary-Minority Free-Energy Adaptive Screening}}
\author{%
\begin{tabular}{c@{\hspace{2.4em}}c}
Wenyan Xu\textsuperscript{1,*} & Alizer Wong\textsuperscript{2,3,*} \\
\texttt{3223004777@mail2.gdut.edu.cn} & \texttt{aliiiiezer@gmail.com}
\end{tabular}\\[0.35em]
\small \textsuperscript{1}Guangdong University of Technology \\
\small \textsuperscript{2}School of Computer Science, Peking University \\
\small \textsuperscript{3}ManXis \\
\small \textsuperscript{*}Equal contribution
}
\date{}

\begin{document}
\pagestyle{plain}
\maketitle
\vspace{-7mm}

\begin{abstract}
Vision Transformers process spatially redundant tokens efficiently only when coarse token summaries preserve the evidence required by exponential attention aggregation. We identify a boundary-minority underestimation failure in which a spatially small, high-response region contributes dominant Gibbs mass while remaining nearly invisible to a block mean. We formalize the failure through the discrepancy between normalized log-mean-exp free energy and mean summarization, prove that minority Gibbs mass can remain non-vanishing as its spatial support and mean contribution vanish, and characterize the limitations of finite-order moment corrections. Building on the resulting analysis, we introduce Boundary-Minority Free-Energy Adaptive Screening (BMFA), which constructs a hierarchical piecewise-constant approximation and recursively refines blocks according to a computable lower-bound increment of local free energy. Controlled synthetic tests, COCO and LVIS diagnostic probes, closed-loop DeiT-Tiny evaluations, and ImageNet-1K experiments establish a consistent evidence chain. BMFA reduces the mean synthetic underestimate from $2.582$ to $0.261$ at a $5.794\%$ leaf ratio, lowers the COCO image-edge mean gap from $2.254$ to $0.526$, and preserves $71.520\%$ ImageNet Top-1 accuracy at a $55.861\%$ leaf ratio. The current prototype evaluates selection quality after full QK computation; the reported leaf ratio therefore characterizes representation granularity rather than verified sparse-kernel speedup.
\end{abstract}

\section{Introduction}
\label{sec:introduction}

\begin{figure}[!t]
  \centering
  \includegraphics[width=\columnwidth]{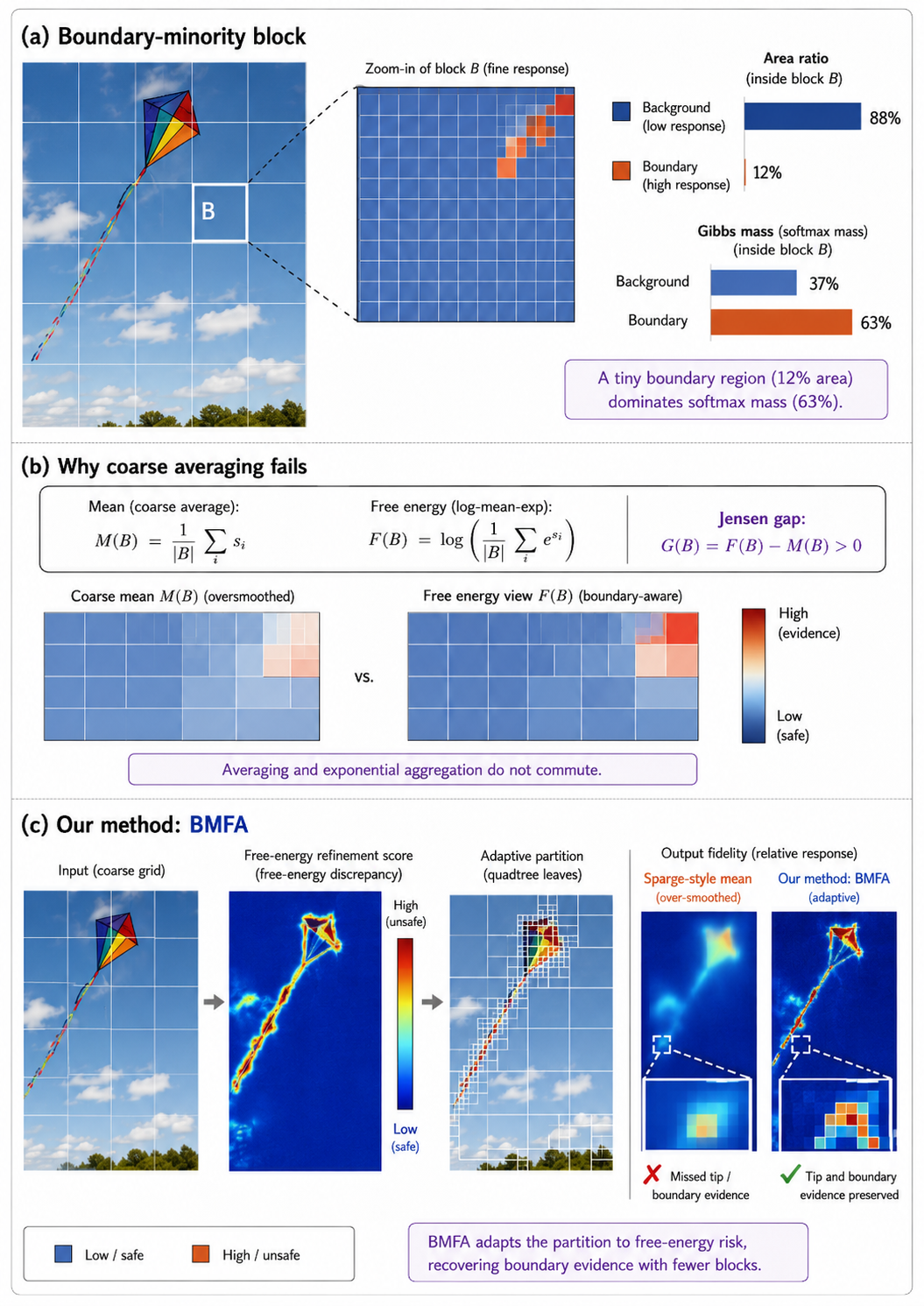}
  \caption{Motivation of BMFA. Top: a coarse block may contain a spatially small boundary subset whose area ratio is low but whose Gibbs mass is dominant. Middle: coarse averaging suppresses the boundary-minority evidence because averaging and exponential aggregation do not commute, producing a positive free-energy gap. Bottom: \textbf{Our method: BMFA} computes a free-energy refinement score, recursively refines only the high-risk regions, and preserves boundary evidence more faithfully than Sparge-style mean summarization.}
  \label{fig:motivation_single}
\end{figure}

Visual representation models have become foundational components of image classification, object detection, semantic segmentation, and vision-language understanding. Vision Transformer (ViT) models introduced by \citet{dosovitskiy2021vit}, hierarchical architectures such as Swin Transformer from \citet{liu2021swin}, dense prediction systems such as Segmenter from \citet{strudel2021segmenter}, and vision-language encoders such as CLIP from \citet{radford2021clip} demonstrate that token-based global interaction can model long-range dependencies within a unified architecture. Nevertheless, the evidence required for visual decisions is not distributed uniformly across an image. Large background regions and object interiors often contain substantial redundancy, whereas object boundaries, elongated structures, small parts, and small objects occupy limited spatial support while carrying decisive evidence for category, shape, or instance assignment. Efficient visual inference must therefore remove redundant computation without suppressing spatially rare but discriminative responses.

A ViT partitions an image into a sequence of patch tokens and repeatedly performs token interaction across Transformer layers. For a sequence of length $N$, standard self-attention constructs an $N\times N$ correlation matrix, yielding quadratic time and memory complexity in $N$; token-wise feed-forward networks and intermediate activations additionally grow linearly with sequence length. Higher image resolution, additional video frames, and multi-scale feature maps consequently make token count a primary constraint on throughput, memory, and energy consumption. Although the data-efficient training strategy of \citet{touvron2021deit} substantially reduced the training burden of visual Transformers, repeated processing of similar tokens remains costly during high-resolution inference. The resulting tension has motivated dynamic token reduction and selective attention computation under task-performance constraints.

Existing efficient visual Transformers follow several complementary directions. DynamicViT from \citet{rao2021dynamicvit} learns token importance and progressively removes redundant positions; EViT from \citet{liang2022evit} reorganizes non-critical tokens according to class-token attention; and A-ViT from \citet{yin2022avit} extends adaptive computation time to spatial tokens. Alternative approaches preserve information through compact representations rather than irreversible deletion. TokenLearner from \citet{ryoo2021tokenlearner} extracts a small set of adaptive tokens through spatial weighting, ToMe from \citet{bolya2023tome} progressively merges tokens according to feature similarity, and STViT from \citet{chang2023stvit} represents local clusters with semantic tokens. Recent studies employ richer structural criteria. Zero-TPrune from \citet{wang2024zerotprune} combines graph-based importance and similarity for zero-shot pruning; ToFu from \citet{kim2024tofu} unifies pruning and merging while mitigating representation shifts; MCTF from \citet{lee2024mctf} jointly considers similarity, information content, and accumulated fusion size; ALGM from \citet{norouzi2024algm} performs local-to-global merging; PiToMe from \citet{tran2024pitome} preserves informative tokens through an energy score; and SAD-TM from \citet{xie2026sadtm} incorporates visual saliency into training-free merging. Beyond token reduction, SparQ Attention from \citet{ribar2024sparq}, QUEST from \citet{tang2024quest}, and SpargeAttention from \citet{zhang2025spargeattn} reduce attention computation through query-dependent retrieval, block-level metadata, or online attention prediction, whereas Native Sparse Attention from \citet{yuan2025nsa} combines coarse compression, fine-grained selection, and hardware-aligned training. Despite substantial differences in learning strategy, compression target, and implementation, many efficient methods rely on importance scores, aggregated tokens, local representatives, or prescribed budgets to summarize collections of original responses. Such summaries are effective when local responses are nearly homogeneous, but a criterion aligned with softmax nonlinearity is generally unavailable for determining whether a coarse summary remains reliable under strong within-block heterogeneity.

The missing criterion becomes particularly consequential for blocks that cross object boundaries. A patch block may contain a large background region and a small foreground fragment, or may jointly cover an object interior, a contour, and an elongated appendage. When most positions have low logits and a small boundary subset has high logits, a block mean or a representative dominated by the majority pattern can classify the entire region as low response. Softmax, however, exponentially weights logits; a spatially small subset can consequently contribute most of the block-level attention mass. We refer to the suppression of minority high-response evidence by majority low-response regions under coarse aggregation as \emph{boundary-minority underestimation}. Boundaries provide a common spatial manifestation, but the underlying failure depends on strong within-block response heterogeneity rather than explicit edge labels. The central question is therefore not whether a region can be compressed in general, but whether a local statistic consistent with exponential aggregation can reveal when a coarse summary fails.

We analyze the failure directly from attention normalization. For concise notation in the introduction, the temperature is set to $\tau=1$; Section~\ref{sec:preliminaries} provides the general-temperature formulation. For a local block $B$ with scores $\{s_i\}_{i\in B}$, define the normalized log-partition value
\begin{equation}
F(B)=\log\!\left(\frac{1}{|B|}\sum_{i\in B}e^{s_i}\right),
\label{eq:free_energy}
\end{equation}
and the mean-based coarse summary
\begin{equation}
M(B)=\frac{1}{|B|}\sum_{i\in B}s_i.
\label{eq:block_mean}
\end{equation}
Convexity of the exponential function and Jensen's inequality yield
\begin{equation}
G(B)=F(B)-M(B)\geq 0,
\label{eq:jensen_gap}
\end{equation}
where $G(B)$ denotes the local free-energy gap. Equation~\eqref{eq:jensen_gap} shows that mean summarization is not an unbiased substitute for softmax-related mass; the approximation systematically underestimates exponential aggregation, and the magnitude depends on the internal score distribution rather than the mean alone.

To isolate the minority mechanism, consider a binary local model in which a majority fraction $1-\alpha$ has score $0$ and a minority fraction $\alpha$ has logit advantage $\Delta>0$. The normalized minority mass is
\begin{equation}
\rho(\alpha,\Delta)=
\frac{\alpha e^{\Delta}}{(1-\alpha)+\alpha e^{\Delta}}.
\label{eq:minority_mass}
\end{equation}
Let $\Delta_{\alpha}=\log[(1-\alpha)/\alpha]+c$, where $c$ is independent of $\alpha$. The minority mass becomes
\begin{equation}
\rho(\alpha,\Delta_{\alpha})=
\frac{e^c}{1+e^c},
\label{eq:nonvanishing_mass}
\end{equation}
which remains non-vanishing as $\alpha\rightarrow0$. By contrast, the mean contribution satisfies
\begin{equation}
\lim_{\alpha\rightarrow0}\alpha\Delta_{\alpha}=0.
\label{eq:mean_blindness}
\end{equation}
The minority support can therefore vanish while retaining a fixed fraction of attention mass, even though its contribution to the block mean disappears. Under the same construction, $G(B)$ converges to the non-zero constant $\log(1+e^c)$. Boundary-minority underestimation is consequently a structural outcome of the non-commutativity between averaging and exponential normalization rather than an incidental effect of noise. Section~\ref{sec:theoretical_motivation} formalizes the limiting argument and connects the gap to attention-distribution fidelity.

A variance correction can partially reduce mean underestimation, but cannot provide a uniformly reliable minority-response criterion. Let $x=s-M(B)$. Equation~\eqref{eq:free_energy} can be written as $F(B)=M(B)+\log\mathbb{E}[e^x]$, whose cumulant expansion contains the variance together with skewness-, kurtosis-, and higher-order terms. The approximation $M(B)+\operatorname{Var}(s)/2$ is locally appropriate for weak, nearly symmetric fluctuations; high-order cumulants remain material for skewed or multimodal distributions formed by rare high responses. Fixed-depth multi-scale partitioning reduces local mixing but expends resolution in homogeneous regions and can remain insufficient in complex regions. Complete preservation of all candidate edge positions provides a high-recall upper bound but requires external spatial priors or near-original granularity. Effective screening must therefore align the error measure with exponential normalization and allocate spatial resolution according to local approximation error.

Motivated by the analysis, we introduce Boundary-Minority Free-Energy Adaptive Screening (BMFA). BMFA uses the local free-energy gap as the theoretical target and a hierarchical lower-bound increment as a computable refinement signal. A candidate block is summarized by its mean, while a finer estimate is obtained from the weighted log-sum-exp of child-block means. If the discrepancy does not exceed a threshold $\varepsilon$, the block remains a coarse leaf; if the discrepancy exceeds $\varepsilon$ and the maximum depth has not been reached, BMFA recursively subdivides the block. The resulting sample- and location-adaptive hierarchy preserves coarse representations over homogeneous backgrounds and assigns finer resolution to boundaries, small parts, and unresolved high-response minorities. Unlike criteria based solely on mean response, prescribed depth, or fixed retention ratio, the refinement decision follows the nonlinear approximation error characterized by Eq.~\eqref{eq:jensen_gap}. Section~\ref{sec:methodology} defines the lower-bound refinement score, establishes its limitations, and presents the recursive construction.

Figure~\ref{fig:motivation_single} summarizes the motivating failure and the corresponding BMFA remedy on a real image. The high-response subsets cover only $12.5\%$ and $10.5\%$ of the corresponding local blocks, yet contribute $97.6\%$ and $97.1\%$ of the within-block Gibbs mass and induce free-energy gaps of $2.923$ and $2.854$. Small spatial support therefore does not imply negligible exponential mass, and the local mean does not reflect the actual contribution of minority responses to softmax normalization. Instance masks and boundaries are used only to construct the diagnostic image-edge probe and to visualize the failure; neither quantity is provided to BMFA during inference.

We evaluate BMFA through an evidence chain spanning mechanism diagnosis and model-level behavior. In controlled synthetic tests, mean-based screening produces an average underestimate of $2.582$, whereas BMFA reduces the value to $0.261$ at a relative leaf ratio of $5.794\%$. On $36{,}149$ COCO instances from \citet{lin2014coco}, BMFA reduces the image-edge mean underestimate from $2.254$ to $0.526$ at a $30.810\%$ leaf ratio. On $20{,}000$ LVIS instances from \citet{gupta2019lvis}, the corresponding reduction is from $2.277$ to $0.596$, indicating that the failure is not dataset-specific. Closed-loop DeiT-Tiny experiments show that BMFA with $\varepsilon=0.8$ attains $75.98\%$ Top-1 agreement with the full model at a granularity comparable to random $25\%$ retention, which reaches only $61.54\%$ agreement; the corresponding KL divergences are $0.149$ and $0.406$. On the ImageNet-1K validation set introduced by \citet{deng2009imagenet}, BMFA with $\varepsilon=0.5$ reaches $71.520\%$ Top-1 accuracy at a $55.861\%$ leaf ratio, only $0.564$ percentage points below the full DeiT-Tiny accuracy of $72.084\%$ and above the $69.148\%$ accuracy of a similarly sized random control. The results demonstrate that fidelity depends not only on retained granularity but also on locating minority responses concealed by coarse summaries.

Our contributions are threefold. First, we formalize boundary-minority underestimation through the non-commutativity of log-mean-exp aggregation and mean summarization. The resulting analysis proves that a minority region can retain non-vanishing softmax mass as both its area and mean contribution vanish, thereby identifying a structural failure condition for coarse token and attention summaries. Second, we introduce BMFA, which uses a local Jensen/free-energy objective and a hierarchical lower-bound increment to recursively allocate spatial resolution to blocks containing unresolved minority responses. Third, we establish a connected empirical validation across synthetic data, COCO, LVIS, closed-loop attention replacement, and ImageNet-1K. The current evaluation isolates screening quality and output fidelity: leaf ratio denotes relative representation granularity and is not equated with FLOPs or throughput improvements that have not been verified by a sparse kernel.

\section{Related Work}
\label{sec:related_work}

\subsection{Visual Token Pruning, Merging, and Adaptive Compression}

Visual token reduction shortens Transformer sequences through pruning, merging, or learned compact representations. Early methods remove redundant tokens using dynamic importance prediction, class-token attention, or adaptive computation time \citep{rao2021dynamicvit,liang2022evit,yin2022avit}, whereas learned extraction and similarity-based methods compress multiple tokens into a smaller set of representatives \citep{ryoo2021tokenlearner,bolya2023tome,chang2023stvit}. Recent work introduces richer structural criteria. \citet{wang2024zerotprune} model attention relationships as a graph and combine weighted PageRank importance with similarity for zero-shot pruning. \citet{kim2024tofu} select between pruning and merging according to local model responses to input interpolation and use MLERP to mitigate norm distortion caused by mean fusion. \citet{lee2024mctf} jointly consider token similarity, information content, and accumulated fusion size, while one-step-ahead attention estimates importance in the subsequent layer.

Spatial structure and information preservation motivate additional designs. \citet{norouzi2024algm} merge similar tokens within shallow local windows before performing global merging at intermediate layers, thereby accommodating the spatial-detail requirements of semantic segmentation. \citet{tran2024pitome} use an energy score to identify large clusters of similar tokens as preferred merging candidates while preserving isolated informative tokens, and analyze information retention from a spectral perspective. \citet{xie2026sadtm} combine attention-space semantic relevance with latent patch saliency and defer merging according to early-layer class-attention changes. These methods primarily determine which tokens should be removed, paired, or fused. BMFA addresses a complementary question: once multiple original responses have been represented by a local coarse summary, does the summary still approximate the exponential mass required by softmax? The free-energy gap in Eq.~\eqref{eq:jensen_gap} diagnoses distortion in mean-based local approximation, while the hierarchical construction in Section~\ref{sec:methodology} restricts refinement to blocks containing unresolved minority responses.

\subsection{Sparse Attention and Hierarchical Coarse-to-Fine Screening}

A separate line of work reduces the query-key interactions evaluated within the attention matrix. \citet{ribar2024sparq} use query-dependent selective history retrieval to reduce KV-cache bandwidth during long-context inference. \citet{tang2024quest} maintain per-dimension extrema for each KV-cache page and combine the metadata with the current query to estimate page criticality, loading only high-ranking candidates. The resulting design exemplifies coarse-to-fine screening through inexpensive block metadata. \citet{zhang2025spargeattn} employ a two-stage online filter: a fast attention-map prediction skips selected matrix multiplications, followed by a softmax-aware filter that removes additional low-contribution computation.

Native Sparse Attention from \citet{yuan2025nsa} incorporates dynamic hierarchical sparsity into end-to-end training and combines coarse-grained token compression with fine-grained selection in a hardware-aligned framework. Prior methods establish the value of block screening, hierarchical representations, and dynamic budgets for efficient attention, but candidate regions are generally determined by predicted importance, extrema-based bounds, top-$k$ budgets, or learned selectors. BMFA instead characterizes a theoretical blind spot of coarse screening itself. When majority low responses and minority high responses coexist within a block, the mean can become uninformative while log-mean-exp mass remains substantial. The local Jensen/free-energy analysis therefore provides a lower-bound refinement signal for approximation failure rather than a generic saliency score. The current prototype evaluates the effect of the signal on spatial selection and model-output fidelity; it does not interpret leaf ratio as sparse-kernel speedup.

\section{Preliminaries}
\label{sec:preliminaries}

\subsection{Attention Scores, Spatial Indices, and Assumptions}
\label{subsec:attention_setting}

Consider a fixed Transformer layer, attention head, and query token. Let $\Omega\subset\mathbb{Z}^{2}$ denote the finite index set of spatial key tokens, with cardinality $N=|\Omega|$; non-spatial tokens, including a classification token, are excluded from the subsequent spatial partition. The query, key, and value vectors are denoted by $\bm q\in\mathbb{R}^{d_h}$, $\bm k_i\in\mathbb{R}^{d_h}$, and $\bm v_i\in\mathbb{R}^{d_v}$, respectively, for $i\in\Omega$. Given a temperature parameter $\tau>0$, the scaled dot-product score, full attention probability, and attention output are
\begin{align}
 s_i&=\frac{\bm q^{\top}\bm k_i}{\sqrt{d_h}},
 &p_i&=\frac{\exp(s_i/\tau)}{\sum_{j\in\Omega}\exp(s_j/\tau)},
 &\bm o&=\sum_{i\in\Omega}p_i\bm v_i.
\label{eq:attention_definitions}
\end{align}
Standard scaled dot-product attention corresponds to $\tau=1$. Retaining an explicit temperature separates the limiting behaviors of smooth and sharply concentrated attention. Equation~\eqref{eq:attention_definitions} follows the standard Transformer formulation of \citet{vaswani2017attention}.

The analysis requires only the following assumptions. First, $\Omega$ is finite and every $s_i$ is a finite real number, so all exponential and logarithmic expressions are well defined. Second, a spatial block $B\subseteq\Omega$ is a non-empty rectangular index set with cardinality $n_B=|B|$. Third, only the attention-output perturbation bound assumes a constant $V_{\max}<\infty$ such that $\|\bm v_i\|_2\leq V_{\max}$. Except for the explicitly specified two-point distribution in the binary minority model, no Gaussianity, symmetry, independence, or identical-distribution assumption is imposed on within-block scores.

To unify finite blocks and idealized probability models, let $U_B$ be an index sampled uniformly from $B$ and define $S_B=s_{U_B}$. The arithmetic and exponential averages over a finite block equal $\mathbb E[S_B]$ and $\mathbb E[\exp(S_B/\tau)]$, respectively. Results stated for random variables therefore apply directly to the empirical score distribution of a finite block.

\subsection{Local Free Energy, Mean Summarization, and Conditional Gibbs Distributions}
\label{subsec:local_free_energy}

For any non-empty block $B$, define the unnormalized block Gibbs mass, normalized free energy, mean summary, and exact free-energy gap as
\begin{align}
 Z_\tau(B)&=\sum_{i\in B}\exp(s_i/\tau),
 &F_\tau(B)&=\tau\log\!\left(\frac{Z_\tau(B)}{n_B}\right),\notag\\
 M(B)&=\frac{1}{n_B}\sum_{i\in B}s_i,
 &G_\tau(B)&=F_\tau(B)-M(B).
\label{eq:block_quantities}
\end{align}
The factor $n_B^{-1}$ removes the trivial logarithmic offset associated with block size, ensuring that $F_\tau(B)$ and $M(B)$ coincide for constant-score blocks. The conditional Gibbs distribution and uniform distribution on $B$ are
\begin{equation}
 p_i^B=\frac{\exp(s_i/\tau)}{Z_\tau(B)},
 \qquad
 u_i^B=\frac{1}{n_B},
 \qquad i\in B.
\label{eq:block_distributions}
\end{equation}
Throughout the paper, the term \emph{free-energy gap} refers exclusively to $G_\tau(B)$ in Eq.~\eqref{eq:block_quantities}; the term is not used as a generic label for prediction discrepancies or arbitrary numerical differences.

\subsection{Hierarchical Partitions and Piecewise-Constant Approximation}
\label{subsec:hierarchical_partition}

A spatial partition $\mathcal P(B)=\{C_1,\ldots,C_m\}$ consists of pairwise disjoint non-empty sub-blocks whose union equals $B$. Let $w_r=n_{C_r}/n_B$ denote the relative weight of $C_r$. A quadtree recursively divides a splittable block into at most four rectangular children; the classical data structure follows the formulation of \citet{finkel1974quadtree}. For a partition $\mathcal P(B)$, define the hierarchical free-energy proxy induced by child-block means as
\begin{equation}
 \Phi_\tau(B;\mathcal P)
 =\tau\log\!\left[
 \sum_{C\in\mathcal P(B)}
 \frac{n_C}{n_B}\exp\!\left(\frac{M(C)}{\tau}\right)
 \right].
\label{eq:partition_proxy}
\end{equation}
When $\mathcal P(B)=\{B\}$, the proxy equals $M(B)$. When every partition element contains a single token, the proxy equals $F_\tau(B)$.

Let $\mathcal L(\mathcal T)$ denote the leaf set of a quadtree $\mathcal T$. For every $i\in L$, with $L\in\mathcal L(\mathcal T)$, define the piecewise-constant score approximation $\widetilde s_i=M(L)$. The corresponding globally normalized free energy and attention distribution are
\begin{align}
 \widetilde F_\tau(\mathcal T)
 &=\tau\log\!\left[
 \frac{1}{N}\sum_{L\in\mathcal L(\mathcal T)}
 n_L\exp\!\left(\frac{M(L)}{\tau}\right)
 \right],\notag\\
 q_i^{\mathcal T}
 &=\frac{\exp(\widetilde s_i/\tau)}
 {\sum_{j\in\Omega}\exp(\widetilde s_j/\tau)}.
\label{eq:tree_approximation}
\end{align}
The leaf ratio is defined as $\lambda(\mathcal T)=|\mathcal L(\mathcal T)|/N$. The ratio measures representation granularity relative to the single-token grid. Without a sparse kernel and end-to-end FLOPs accounting, $\lambda(\mathcal T)$ is not equivalent to realized acceleration.

\subsection{Ideal Objective for Free-Energy-Constrained Coarsening}
\label{subsec:safe_coarsening}

The ideal safe-coarsening problem minimizes the number of leaves while constraining global free-energy error:
\begin{equation}
 \min_{\mathcal T\in\mathfrak T_D}|\mathcal L(\mathcal T)|
 \quad\text{s.t.}\quad
 F_\tau(\Omega)-\widetilde F_\tau(\mathcal T)\leq\varepsilon,
\label{eq:ideal_coarsening}
\end{equation}
where $\mathfrak T_D$ is the set of valid spatial trees with maximum depth at most $D$, and $\varepsilon\geq0$ is the admissible error. Equation~\eqref{eq:ideal_coarsening} defines a combinatorial optimization problem. BMFA does not solve the globally optimal tree directly; instead, it recursively constructs an interpretable approximation using local hierarchical lower bounds.

The theory is first established independently for a fixed query-head pair. Sharing a tree across queries or heads requires an explicit aggregation operator, such as the maximum or a high quantile of query-head refinement scores. Different aggregation rules induce different levels of conservativeness; no unspecified sharing strategy is treated as an implicit assumption in the subsequent results.

\section{Theoretical Analysis of Boundary-Minority Underestimation}
\label{sec:theoretical_motivation}

\subsection{Systematic Free-Energy Underestimation by Mean Summarization}
\label{subsec:systematic_underestimation}

\begin{theorem}[Systematic underestimation by the mean]
\label{thm:jensen_underestimation}
For every non-empty block $B$ and every $\tau>0$, the exact free-energy gap satisfies $G_\tau(B)\geq0$. Equality holds if and only if all scores in $\{s_i:i\in B\}$ are equal.
\end{theorem}

\begin{proof}
The function $x\mapsto\exp(x/\tau)$ is strictly convex on $\mathbb R$. Jensen's inequality \citep{boyd2004convex} gives
\begin{equation}
 \frac{1}{n_B}\sum_{i\in B}\exp(s_i/\tau)
 \geq
 \exp\!\left(\frac{1}{n_B}\sum_{i\in B}\frac{s_i}{\tau}\right)
 =\exp\!\left(\frac{M(B)}{\tau}\right).
\label{eq:jensen_step}
\end{equation}
Taking natural logarithms and multiplying both sides by $\tau>0$ yields $F_\tau(B)\geq M(B)$, equivalently $G_\tau(B)\geq0$. Equality in Jensen's inequality for a strictly convex function requires identical inputs, so equality holds exactly when $s_i=M(B)$ for every $i\in B$.
\end{proof}

Theorem~\ref{thm:jensen_underestimation} establishes a deterministic direction for the approximation error: local underestimation does not arise merely from sampling noise or training instability.

\begin{corollary}[Relative-entropy representation of the free-energy gap]
\label{cor:gap_kl_uniform}
Let $u^B$ and $p^B$ be the uniform and conditional Gibbs distributions in Eq.~\eqref{eq:block_distributions}. Then
\begin{equation}
 G_\tau(B)=\tau D_{\mathrm{KL}}(u^B\|p^B).
\label{eq:gap_reverse_kl}
\end{equation}
\end{corollary}

\begin{proof}
Using $p_i^B=\exp(s_i/\tau)/Z_\tau(B)$ and expanding the relative entropy term by term,
\begin{align}
 D_{\mathrm{KL}}(u^B\|p^B)
 &=\sum_{i\in B}\frac{1}{n_B}
 \log\frac{1/n_B}{p_i^B}\notag\\
 &=\log Z_\tau(B)-\log n_B-
 \frac{1}{n_B}\sum_{i\in B}\frac{s_i}{\tau}\notag\\
 &=\frac{F_\tau(B)-M(B)}{\tau}.
\label{eq:gap_kl_derivation}
\end{align}
Multiplication by $\tau$ proves Eq.~\eqref{eq:gap_reverse_kl}.
\end{proof}

Equation~\eqref{eq:gap_reverse_kl} shows that $G_\tau(B)$ measures exactly how far the within-block Gibbs distribution departs from uniformity; the quantity is not a generic dispersion measure disconnected from attention normalization.

\begin{proposition}[Translation invariance and temperature monotonicity]
\label{prop:temperature_properties}
For every $c\in\mathbb R$, replacing all block scores by $s_i+c$ leaves $G_\tau(B)$ unchanged. If the block scores are not constant, then $G_\tau(B)$ is strictly decreasing in $\tau$ and satisfies
\begin{equation}
 \lim_{\tau\to\infty}G_\tau(B)=0,
 \qquad
 \lim_{\tau\to0^+}G_\tau(B)=\max_{i\in B}s_i-M(B).
\label{eq:temperature_limits}
\end{equation}
\end{proposition}

\begin{proof}[Proof sketch]
Translation invariance follows because $F_\tau(B)$ and $M(B)$ change by the same additive offset. Differentiating with respect to temperature reduces the derivative to $-D_{\mathrm{KL}}(p^B\|u^B)$, which establishes strict monotonicity for non-constant scores. The high-temperature limit follows from a second-order expansion of the centered exponential moment, whereas the low-temperature limit follows from isolating the maximal term in log-sum-exp. The complete derivation appears in Supplementary Section~\ref{app:proof_boundary_minority}.
\end{proof}

Proposition~\ref{prop:temperature_properties} delineates the relevant regime. High-temperature attention approaches uniform aggregation and makes mean summarization comparatively accurate. Low-temperature or high-contrast attention approaches maximum selection and increases the risk that minority high responses induce substantial underestimation.

\subsection{Binary Boundary-Minority Model and Mass Reversal}
\label{subsec:binary_minority_model}

To isolate the roles of spatial proportion and response advantage, consider the two-point random variable $S_{\alpha,\Delta}$, which equals $0$ with probability $1-\alpha$ and $\Delta$ with probability $\alpha$, where $0<\alpha\leq1/2$ and $\Delta>0$. The model does not assume that real scores are exactly binary; the two states provide the minimal analyzable representation of majority low responses and minority high responses.

\begin{lemma}[Free energy, mean, and Gibbs mass in the binary model]
\label{lem:binary_exact}
The mean, free-energy gap, and exponentially tilted probability of the minority state are
\begin{align}
 M_{\alpha,\Delta}&=\alpha\Delta,\notag\\
 G_\tau(\alpha,\Delta)
 &=\tau\log\!\left[(1-\alpha)+\alpha e^{\Delta/\tau}\right]-\alpha\Delta,\notag\\
 \rho_\tau(\alpha,\Delta)
 &=\frac{\alpha e^{\Delta/\tau}}
 {(1-\alpha)+\alpha e^{\Delta/\tau}}.
\label{eq:binary_closed_forms}
\end{align}
\end{lemma}

\begin{proof}
The mean follows directly from the two-point distribution. The exponential moment is $\mathbb E[e^{S_{\alpha,\Delta}/\tau}]=(1-\alpha)+\alpha e^{\Delta/\tau}$; substituting the moment into $F_\tau=\tau\log\mathbb E[e^{S/\tau}]$ and subtracting the mean yields the second expression. The minority probability after exponential tilting equals its unnormalized mass $\alpha e^{\Delta/\tau}$ divided by the total mass, which yields the third expression.
\end{proof}

\begin{proposition}[Critical condition for minority-mass dominance]
\label{prop:minority_threshold}
For a target mass $r\in(0,1)$, the minority state satisfies $\rho_\tau(\alpha,\Delta)\geq r$ if and only if
\begin{equation}
 \Delta\geq
 \tau\log\frac{r(1-\alpha)}{(1-r)\alpha}.
\label{eq:minority_threshold}
\end{equation}
In particular, the minority state carries at least one-half of the Gibbs mass if and only if $\Delta\geq\tau\log[(1-\alpha)/\alpha]$.
\end{proposition}

\begin{proof}
The denominator of $\rho_\tau(\alpha,\Delta)$ is strictly positive. Cross-multiplication of $\rho_\tau(\alpha,\Delta)\geq r$ gives $(1-r)\alpha e^{\Delta/\tau}\geq r(1-\alpha)$. Dividing by the positive factor $(1-r)\alpha$ and taking natural logarithms produces Eq.~\eqref{eq:minority_threshold}. Every transformation is monotone and reversible, so the condition is both necessary and sufficient.
\end{proof}

Proposition~\ref{prop:minority_threshold} separates minority area from minority attention mass: once the logit advantage exceeds the logarithmic threshold, $\alpha<1/2$ does not prevent the minority state from dominating exponential mass.

\begin{theorem}[Asymptotic mean blindness for a boundary minority]
\label{thm:asymptotic_blindness}
Fix $c\in\mathbb R$ and $\tau>0$, and define $\Delta_\alpha=\tau[\log((1-\alpha)/\alpha)+c]$. As $\alpha\to0^+$,
\begin{align}
 \rho_\tau(\alpha,\Delta_\alpha)
 &\longrightarrow\frac{e^c}{1+e^c},
 &M_{\alpha,\Delta_\alpha}&\longrightarrow0,\notag\\
 G_\tau(\alpha,\Delta_\alpha)
 &\longrightarrow\tau\log(1+e^c)>0.&&
\label{eq:asymptotic_blindness}
\end{align}
\end{theorem}

\begin{proof}[Proof sketch]
Substitution of $\Delta_\alpha$ into Eq.~\eqref{eq:binary_closed_forms} transforms the unnormalized minority mass into $(1-\alpha)e^c$, making the Gibbs mass independent of $\alpha$. The mean limit follows from $\alpha\log(1/\alpha)\to0$, while the free-energy limit follows by extracting the common factor $1-\alpha$. Step-by-step algebra and the limiting argument appear in Supplementary Section~\ref{app:proof_boundary_minority}.
\end{proof}

Theorem~\ref{thm:asymptotic_blindness} captures the central mechanism of boundary-minority underestimation: spatial support and mean contribution can vanish simultaneously while exponential mass and the free-energy gap remain non-zero.

\subsection{Boundary Regimes and Necessary Conditions}
\label{subsec:boundary_regimes}

\begin{proposition}[Boundary regimes of the binary model]
\label{prop:binary_boundaries}
The binary model satisfies the following properties. The free-energy gap equals zero when $\Delta=0$ or $\alpha\in\{0,1\}$. For fixed $\Delta$ and $\alpha\to0^+$,
\begin{equation}
 G_\tau(\alpha,\Delta)
 =\alpha\left[\tau(e^{\Delta/\tau}-1)-\Delta\right]+O(\alpha^2)
 \longrightarrow0.
\label{eq:fixed_delta_small_alpha}
\end{equation}
For fixed $\alpha\in(0,1)$ and $\Delta\to\infty$,
\begin{equation}
 G_\tau(\alpha,\Delta)
 =(1-\alpha)\Delta+\tau\log\alpha+o(1).
\label{eq:large_delta_gap}
\end{equation}
Moreover, for fixed $\alpha\in(0,1)$, $G_\tau(\alpha,\Delta)$ is strictly increasing and strictly convex for $\Delta>0$.
\end{proposition}

\begin{proof}[Proof sketch]
The small-area limit follows from a local expansion of $\log(1+x)$, and the large-advantage limit follows by factoring $\alpha e^{\Delta/\tau}$ from log-sum-exp. The signs of the first and second derivatives with respect to $\Delta$ establish monotonicity and convexity. The complete proof appears in Supplementary Section~\ref{app:proof_boundary_minority}.
\end{proof}

Equation~\eqref{eq:fixed_delta_small_alpha} prevents an overly broad conclusion: small spatial support alone does not require refinement. A minority region retains substantial Gibbs mass only when its response advantage grows sufficiently as its area shrinks.

\subsection{Validity and Failure of Finite-Order Moment Corrections}
\label{subsec:moment_limitations}

Let $X_B=S_B-M(B)$ denote the centered block score, and let $\kappa_r(B)$ denote its $r$-th cumulant. Because scores in a finite block are bounded, the moment-generating function is finite over the real line and the cumulant-generating function is analytic in a neighborhood of the origin.

\begin{lemma}[Cumulant expansion in the weak-contrast regime]
\label{lem:cumulant_expansion}
Scale all centered scores in the block by $\lambda$. As $\lambda\to0$,
\begin{equation}
 G_\tau(\lambda S_B)
 =\frac{\lambda^2}{2\tau}\kappa_2(B)
 +\frac{\lambda^3}{6\tau^2}\kappa_3(B)
 +\frac{\lambda^4}{24\tau^3}\kappa_4(B)
 +O(\lambda^5).
\label{eq:cumulant_expansion}
\end{equation}
\end{lemma}

\begin{proof}[Proof sketch]
Translation invariance removes the block mean, after which $G_\tau$ equals $\tau$ times the cumulant-generating function of the centered variable evaluated at $\lambda/\tau$. Expanding the generating function around the origin and using $\kappa_1(B)=0$ gives the result. The complete expansion appears in Supplementary Section~\ref{app:proof_moments}.
\end{proof}

Lemma~\ref{lem:cumulant_expansion} shows that $M(B)+\operatorname{Var}(S_B)/(2\tau)$ is only a second-order local approximation in weak-contrast regions. The correction lacks a uniform guarantee when skewness, kurtosis, or higher-order cumulants remain non-negligible.

\begin{theorem}[Asymptotic blindness of finite-order moment corrections]
\label{thm:finite_moment_blindness}
Consider $S_{\alpha,\Delta_\alpha}$ from Theorem~\ref{thm:asymptotic_blindness}. For every fixed integer $r\geq1$, its $r$-th central absolute moment $\mu_r(\alpha)=\mathbb E|S_{\alpha,\Delta_\alpha}-M_{\alpha,\Delta_\alpha}|^r$ satisfies
\begin{equation}
 \mu_r(\alpha)
 =O\!\left(\alpha[\log(1/\alpha)]^r\right)
 \longrightarrow0.
\label{eq:central_moment_vanish}
\end{equation}
Furthermore, let $C:\mathbb R^{m-1}\to\mathbb R$ be continuous at the origin and satisfy $C(0,\ldots,0)=0$. Every correction of the form $M_{\alpha,\Delta_\alpha}+C(\mu_2,\ldots,\mu_m)$, depending only on a fixed finite set of central moments, converges to zero, whereas the exact free energy converges to $\tau\log(1+e^c)>0$.
\end{theorem}

\begin{proof}[Proof sketch]
The central absolute moment of the binary model separates into contributions from the majority and minority states. The growth rate $\Delta_\alpha=O(\log(1/\alpha))$ bounds both contributions by $O(\alpha[\log(1/\alpha)]^r)$. The finite-dimensional moment vector therefore converges to the origin, and continuity of $C$ yields the inconsistency. Supplementary Section~\ref{app:proof_moments} provides the complete proof.
\end{proof}

\begin{corollary}[Inconsistency of the mean-variance correction]
\label{cor:variance_inconsistent}
Under the construction of Theorem~\ref{thm:asymptotic_blindness},
$M_{\alpha,\Delta_\alpha}+\operatorname{Var}(S_{\alpha,\Delta_\alpha})/(2\tau)\to0$, whereas $F_\tau(S_{\alpha,\Delta_\alpha})\to\tau\log(1+e^c)$.
\end{corollary}

The result does not deny the usefulness of variance correction for finite samples and weakly skewed distributions. Instead, no fixed-order correction that is continuous near zero moments can uniformly cover the boundary-minority limit in which area vanishes but exponential mass remains non-zero.

\subsection{Partitioned Free-Energy Decomposition and Hierarchical Lower Bounds}
\label{subsec:partition_theory}

\begin{lemma}[Partition identity for free energy]
\label{lem:partition_identity}
For every partition $\mathcal P(B)$, the exact free energy satisfies
\begin{equation}
 F_\tau(B)
 =\tau\log\!\left[
 \sum_{C\in\mathcal P(B)}
 \frac{n_C}{n_B}
 \exp\!\left(\frac{F_\tau(C)}{\tau}\right)
 \right].
\label{eq:partition_identity}
\end{equation}
\end{lemma}

\begin{proof}
Pairwise disjointness and $\bigcup_{C\in\mathcal P(B)}C=B$ imply
\begin{align}
 \frac{1}{n_B}\sum_{i\in B}e^{s_i/\tau}
 &=\sum_{C\in\mathcal P(B)}\frac{n_C}{n_B}
 \left(\frac{1}{n_C}\sum_{i\in C}e^{s_i/\tau}\right)\notag\\
 &=\sum_{C\in\mathcal P(B)}\frac{n_C}{n_B}
 e^{F_\tau(C)/\tau}.
\label{eq:partition_sum_decomposition}
\end{align}
Taking the logarithm and multiplying by $\tau$ proves Eq.~\eqref{eq:partition_identity}.
\end{proof}

\begin{proposition}[Two-sided bounds for the hierarchical proxy]
\label{prop:proxy_bounds}
For every partition $\mathcal P(B)$,
\begin{equation}
 M(B)\leq\Phi_\tau(B;\mathcal P)\leq F_\tau(B).
\label{eq:proxy_sandwich}
\end{equation}
Consequently, the refinement increment $\Phi_\tau(B;\mathcal P)-M(B)$ is a non-negative lower bound on the exact free-energy gap $G_\tau(B)$.
\end{proposition}

\begin{proof}
The area-weighted average of child means equals the parent mean: $\sum_C(n_C/n_B)M(C)=M(B)$. Applying Jensen's inequality to the strictly convex function $x\mapsto e^{x/\tau}$ gives
\begin{equation}
 \sum_C\frac{n_C}{n_B}e^{M(C)/\tau}
 \geq
 e^{\sum_C(n_C/n_B)M(C)/\tau}
 =e^{M(B)/\tau}.
\end{equation}
Taking the logarithm and multiplying by $\tau$ yields the left inequality. Theorem~\ref{thm:jensen_underestimation} gives $M(C)\leq F_\tau(C)$ for every child. Monotonicity of the exponential function together with Lemma~\ref{lem:partition_identity} yields the right inequality.
\end{proof}

\begin{theorem}[Monotone consistency under partition refinement]
\label{thm:partition_monotonicity}
Let $\mathcal Q(B)$ refine $\mathcal P(B)$, meaning that every block in $\mathcal Q$ is contained in a block of $\mathcal P$. Then
\begin{equation}
 \Phi_\tau(B;\mathcal P)
 \leq\Phi_\tau(B;\mathcal Q)
 \leq F_\tau(B).
\label{eq:partition_monotonicity}
\end{equation}
The second inequality becomes equality when $\mathcal Q(B)$ is the single-token partition.
\end{theorem}

\begin{proof}[Proof sketch]
Within each coarse block, express the parent mean as the area-weighted average of descendant means and apply Jensen's inequality to the exponential function. Summation over all coarse blocks shows that refinement cannot decrease the proxy free energy. The single-token partition equals the exact free energy by definition. Supplementary Section~\ref{app:proof_partition} contains the complete proof.
\end{proof}

Theorem~\ref{thm:partition_monotonicity} establishes that hierarchical refinement constructs an ordered sequence of lower bounds that converges to the exact free energy, rather than an unordered collection of heuristic estimates.

\begin{theorem}[Recursive residual decomposition of parent and child gaps]
\label{thm:recursive_residual}
For a partition $\mathcal P(B)$, define
\begin{equation}
 \pi_C=
 \frac{(n_C/n_B)e^{M(C)/\tau}}
 {\sum_{R\in\mathcal P(B)}(n_R/n_B)e^{M(R)/\tau}}.
\label{eq:coarse_child_weights}
\end{equation}
The exact free-energy gap satisfies
\begin{align}
 G_\tau(B)
 &=\Phi_\tau(B;\mathcal P)-M(B)\notag\\
 &\quad+\tau\log\!\left[
 \sum_{C\in\mathcal P(B)}
 \pi_Ce^{G_\tau(C)/\tau}
 \right].
\label{eq:recursive_gap_decomposition}
\end{align}
\end{theorem}

\begin{proof}[Proof sketch]
Substitute $F_\tau(C)=M(C)+G_\tau(C)$ into Lemma~\ref{lem:partition_identity} and factor out the common log-sum-exp term formed by child means. The remaining normalized coefficients are exactly $\pi_C$ in Eq.~\eqref{eq:coarse_child_weights}. Subtracting $M(B)$ produces the recursive decomposition. The complete derivation appears in Supplementary Section~\ref{app:proof_partition}.
\end{proof}

Because $\{\pi_C\}$ forms a probability distribution and log-sum-exp lies between its minimum and maximum arguments,
\begin{equation}
 \min_CG_\tau(C)
 \leq G_\tau(B)-\Phi_\tau(B;\mathcal P)+M(B)
 \leq\max_CG_\tau(C).
\label{eq:recursive_gap_bounds}
\end{equation}
Equation~\eqref{eq:recursive_gap_bounds} separates heterogeneity across child blocks from unresolved heterogeneity within children. A one-step refinement score captures only the former.

\subsection{Global Distribution Fidelity and Output Perturbation}
\label{subsec:global_fidelity}

\begin{theorem}[KL equivalence of global free-energy error]
\label{thm:global_kl_equivalence}
Let $p$ be the full attention distribution in Eq.~\eqref{eq:attention_definitions}, and let $q^{\mathcal T}$ be the piecewise-constant approximation in Eq.~\eqref{eq:tree_approximation}. Then
\begin{equation}
 F_\tau(\Omega)-\widetilde F_\tau(\mathcal T)
 =\tau D_{\mathrm{KL}}(q^{\mathcal T}\|p).
\label{eq:global_kl_equivalence}
\end{equation}
\end{theorem}

\begin{proof}[Proof sketch]
Expanding $D_{\mathrm{KL}}(q^{\mathcal T}\|p)$ produces the ratio of partition functions and a score-replacement term within each leaf. The approximate distribution is constant within every leaf, and $M(L)$ is the exact arithmetic mean of the original scores, so the replacement term cancels exactly within each leaf. The remaining term is the difference between the full and approximate log-partition functions. Supplementary Section~\ref{app:proof_global} provides the complete derivation.
\end{proof}

\begin{corollary}[Propagation from local leaf gaps to global error]
\label{cor:local_to_global}
Define the approximate mass assigned to leaf $L$ by
\begin{equation}
 \widetilde\pi_L=
 \frac{n_Le^{M(L)/\tau}}
 {\sum_{R\in\mathcal L(\mathcal T)}n_Re^{M(R)/\tau}}.
\label{eq:leaf_coarse_mass}
\end{equation}
Then
\begin{equation}
 F_\tau(\Omega)-\widetilde F_\tau(\mathcal T)
 =\tau\log\sum_{L\in\mathcal L(\mathcal T)}
 \widetilde\pi_Le^{G_\tau(L)/\tau}.
\label{eq:local_global_decomposition}
\end{equation}
If every leaf satisfies $G_\tau(L)\leq\varepsilon$, then the global free-energy error does not exceed $\varepsilon$.
\end{corollary}

\begin{proof}[Proof sketch]
Group the full partition function by leaves and substitute $F_\tau(L)=M(L)+G_\tau(L)$. The ratio between the full and approximate partition functions becomes a $\widetilde\pi_L$-weighted average of $e^{G_\tau(L)/\tau}$. If every leaf gap is at most $\varepsilon$, the weighted average is at most $e^{\varepsilon/\tau}$. Supplementary Section~\ref{app:proof_global} gives the complete proof.
\end{proof}

Corollary~\ref{cor:local_to_global} controls the exact leaf gap, not a one-step child proxy. The current BMFA refinement score is a lower bound on the exact quantity and cannot by itself imply the certified conclusion without an additional upper bound.

\begin{corollary}[Attention-output perturbation bound]
\label{cor:attention_output_bound}
Assume $\|\bm v_i\|_2\leq V_{\max}$, and define $E_\tau(\mathcal T)=F_\tau(\Omega)-\widetilde F_\tau(\mathcal T)$. Then
\begin{equation}
 \|\bm o-\widetilde{\bm o}\|_2
 \leq V_{\max}\sqrt{\frac{2E_\tau(\mathcal T)}{\tau}},
 \qquad
 \widetilde{\bm o}=\sum_iq_i^{\mathcal T}\bm v_i.
\label{eq:attention_output_bound}
\end{equation}
\end{corollary}

\begin{proof}[Proof sketch]
The value-norm bound and the triangle inequality first give $\|\bm o-\widetilde{\bm o}\|_2\leq V_{\max}\|p-q^{\mathcal T}\|_1$. Pinsker's inequality \citep{sason2015reversepinsker}, followed by Theorem~\ref{thm:global_kl_equivalence}, gives Eq.~\eqref{eq:attention_output_bound}. The complete proof appears in Supplementary Section~\ref{app:proof_global}.
\end{proof}

Corollary~\ref{cor:attention_output_bound} establishes a rigorous connection between free-energy preservation, attention-distribution preservation, and single-layer attention-output stability. Classification accuracy additionally depends on subsequent nonlinear mappings and classification margins; monotonic Top-1 behavior does not follow automatically from the bound.

\subsection{Incompleteness of One-Step Screening and Certified Corrections}
\label{subsec:criterion_limitations}

\begin{proposition}[Incompleteness of the one-step refinement score]
\label{prop:one_step_counterexample}
There exists a non-constant block $B$ and a four-child partition $\mathcal P_1(B)$ such that $\Phi_\tau(B;\mathcal P_1)-M(B)=0$ while $G_\tau(B)>0$.
\end{proposition}

\begin{proof}
Let every child contain equal numbers of $+a$ and $-a$, where $a>0$. Every child mean and the parent mean equal zero, so Eq.~\eqref{eq:partition_proxy} yields $\Phi_\tau(B;\mathcal P_1)=0=M(B)$. The empirical distribution of the entire block also assigns equal probability to $+a$ and $-a$, and therefore
\begin{equation}
 G_\tau(B)
 =\tau\log\!\left[\frac{e^{a/\tau}+e^{-a/\tau}}{2}\right]
 =\tau\log\cosh(a/\tau)>0.
\label{eq:one_step_counterexample}
\end{equation}
A one-step child-mean score consequently fails to reveal within-child heterogeneity when means cancel exactly at the current partition scale.
\end{proof}

Proposition~\ref{prop:one_step_counterexample} requires the one-step quantity to be described as a \emph{lower-bound refinement signal}, rather than an unconditional complete error certificate.

\begin{theorem}[Monotone consistency of multi-level look-ahead lower bounds]
\label{thm:lookahead_consistency}
Let $\mathcal P_h(B)$ be the descendant partition obtained by regular subdivision of $B$ for $h$ levels, and define $\widehat G_\tau^{(h)}(B)=\Phi_\tau(B;\mathcal P_h)-M(B)$. Then
\begin{equation}
 0=\widehat G_\tau^{(0)}(B)
 \leq\widehat G_\tau^{(1)}(B)
 \leq\cdots\leq G_\tau(B).
\label{eq:lookahead_monotonicity}
\end{equation}
If $h$ is sufficiently large that every descendant is a single token, then $\widehat G_\tau^{(h)}(B)=G_\tau(B)$.
\end{theorem}

\begin{proof}[Proof sketch]
Adjacent look-ahead partitions form a sequence of refinements. Theorem~\ref{thm:partition_monotonicity} therefore guarantees that the proxy free energy is non-decreasing and bounded above by the exact free energy. The proxy equals the exact value at the single-token partition. Supplementary Section~\ref{app:proof_certification} provides the full proof.
\end{proof}

Multi-level look-ahead reduces cancellation at the current partition scale but increases statistical access and tree-construction cost. The experiments use the low-cost setting $h=1$; deeper look-ahead is a testable extension rather than an implicit component of the reported implementation.

\begin{theorem}[Range-based upper bound on the free-energy gap]
\label{thm:range_upper_bound}
Define the within-block range $\mathcal R(B)=\max_{i\in B}s_i-\min_{i\in B}s_i$. Then
\begin{equation}
 G_\tau(B)\leq\frac{\mathcal R(B)^2}{8\tau}.
\label{eq:range_upper_bound}
\end{equation}
\end{theorem}

\begin{proof}[Proof sketch]
Apply Hoeffding's lemma \citep{hoeffding1963probability} to the centered variable $X_B=S_B-M(B)$, whose support interval has length $\mathcal R(B)$. Setting the moment-generating-function parameter to $1/\tau$ and multiplying by $\tau$ gives the result. Supplementary Section~\ref{app:proof_certification} contains the complete proof.
\end{proof}

Theorem~\ref{thm:range_upper_bound} and the multi-level lower bound yield $\widehat G_\tau^{(h)}(B)\leq G_\tau(B)\leq\mathcal R(B)^2/(8\tau)$. A block is certified safe when the upper bound is at most $\varepsilon$, and refinement is necessary when the lower bound exceeds $\varepsilon$. If the lower bound does not exceed the threshold while the upper bound does, the available statistics are insufficient for a certified decision. The range bound requires access to within-block extrema. Such access is available in the diagnostic prototype after full score computation, but a future implementation designed to avoid QK computation would require a low-cost proxy or a precomputed bound.

\begin{proposition}[Stability under approximate scores]
\label{prop:score_stability}
Assume that exact and approximate scores satisfy $\max_{i\in B}|s_i-\widehat s_i|\leq\delta$. Let $G_\tau(B;s)$ and $G_\tau(B;\widehat s)$ denote the free-energy gaps computed from the two score sets. Then
\begin{equation}
 |G_\tau(B;s)-G_\tau(B;\widehat s)|\leq2\delta.
\label{eq:score_stability}
\end{equation}
\end{proposition}

\begin{proof}[Proof sketch]
Log-sum-exp is $1$-Lipschitz with respect to the $\ell_\infty$ norm, so the corresponding free energies differ by at most $\delta$. Arithmetic means also differ by at most $\delta$. Applying the triangle inequality to the difference of the two terms gives the $2\delta$ bound. A termwise exponential sandwich proof appears in Supplementary Section~\ref{app:proof_certification}.
\end{proof}

A future screening criterion based on predicted scores must preserve a safety margin of at least $2\delta$ on both sides of the threshold. Unquantified proxy error cannot be ignored.

\section{Adaptive Multi-Scale Screening from Free-Energy Lower Bounds}
\label{sec:methodology}

\subsection{Overview and Design Principles}
\label{subsec:method_overview}

BMFA constructs a piecewise-constant approximation over the two-dimensional spatial layout of attention scores. The design follows three principles. First, the refinement quantity must align with exponential softmax aggregation rather than relying only on a linear mean or an external edge prior. Second, spatial granularity must adapt to local heterogeneity, retaining coarse leaves in homogeneous regions and allocating the leaf budget to regions with unresolved response differences. Third, the threshold and maximum depth must explicitly control tree size, while the theoretical presentation must distinguish refinement triggered by a lower bound from safety certified by an upper bound.

The current formulation constructs an independent tree for each fixed query-head pair. The choice requires no cross-head consistency assumption and exactly matches the results in Sections~\ref{sec:preliminaries}--\ref{sec:theoretical_motivation}. An implementation that shares a tree can compute refinement scores for all participating query-head pairs and aggregate them by a maximum. The conservative maximum guarantees refinement whenever any participating pair triggers a split, although the resulting computation and budget must be reported separately.

Figure~\ref{fig:bmfa_pipeline} connects the local BMFA decision to real inputs and closed-loop output fidelity. The response field provides a computable representation of within-block heterogeneity, and the hierarchical lower-bound increment in Eq.~\eqref{eq:bmfa_refinement_score} determines whether a quadtree node is subdivided. Large leaves remain in homogeneous regions, while additional granularity is concentrated around high-response boundaries. The rightmost panel uses measured closed-loop results from the last four DeiT-Tiny blocks on COCO 5k. At comparable leaf ratios, adaptive selection achieves higher Top-1 agreement than random selection.

\begin{figure*}[t]
  \centering
  \includegraphics[width=0.98\textwidth]{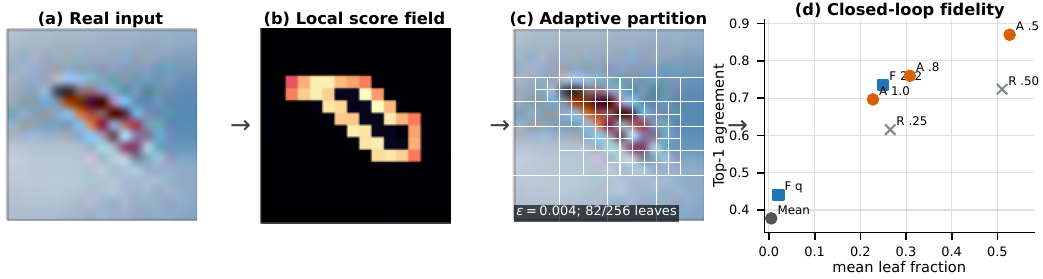}
  \caption{Data-driven overview of BMFA. (a) A kite crop from COCO. (b) The corresponding local response field. (c) Adaptive leaf partition recursively constructed from $\Gamma_1(B)$; the illustrated threshold is $\varepsilon=0.004$, producing $82/256$ leaves. (d) Leaf ratio and Top-1 agreement with the full model in the COCO 5k closed-loop attention experiment. A, R, and F denote adaptive, random, and fixed strategies. Adaptive selection consistently exceeds random retention at comparable granularity, indicating that the gain arises from spatial selection rather than retention quantity alone. Panel (c) is reconstructed from the measured response grid in panel (b) using Algorithm~\ref{alg:bmfa}; panel (d) uses the recorded experimental results.}
  \label{fig:bmfa_pipeline}
\end{figure*}

\subsection{One-Step Hierarchical Free-Energy Refinement Score}
\label{subsec:refinement_score}

Let $\operatorname{ch}(B)$ denote the valid direct children of a splittable block $B$. The one-step BMFA refinement score is
\begin{align}
 \Gamma_\tau(B)
 &=\Phi_\tau(B;\operatorname{ch}(B))-M(B)\notag\\
 &=\tau\log\!\left[
 \sum_{C\in\operatorname{ch}(B)}
 \frac{n_C}{n_B}e^{M(C)/\tau}
 \right]-M(B).
\label{eq:bmfa_refinement_score}
\end{align}
Proposition~\ref{prop:proxy_bounds} guarantees $0\leq\Gamma_\tau(B)\leq G_\tau(B)$. Consequently, $\Gamma_\tau(B)>\varepsilon$ is sufficient evidence that the current coarse representation incurs an error exceeding $\varepsilon$. By contrast, $\Gamma_\tau(B)\leq\varepsilon$ establishes only that heterogeneity across direct children is small; deeper mean cancellation of the form in Proposition~\ref{prop:one_step_counterexample} remains possible. We therefore refer to $\Gamma_\tau$ as a hierarchical free-energy refinement score rather than an unconditional complete error certificate.

Equation~\eqref{eq:bmfa_refinement_score} depends only on the parent mean and direct-child means. Once a score grid is available, a two-dimensional prefix sum permits constant-time mean queries for arbitrary rectangular blocks, making the decision cost for each tree node independent of block area.

\subsection{Recursive Quadtree Construction}
\label{subsec:recursive_algorithm}

Given a threshold $\varepsilon\geq0$ and maximum depth $D$, BMFA processes the root block $\Omega$ recursively. A node becomes a leaf when the maximum depth is reached, the node contains a single token, or no valid subdivision exists. Every remaining node is split when $\Gamma_\tau(B)>\varepsilon$ and retained otherwise.

\begin{algorithm}[t]
\caption{Boundary-Minority Free-Energy Adaptive Screening}
\label{alg:bmfa}
\KwIn{Score grid $\{s_i\}_{i\in\Omega}$, temperature $\tau$, threshold $\varepsilon$, maximum depth $D$}
\KwOut{Leaf partition $\mathcal L(\mathcal T)$}
Initialize $\mathcal Q\leftarrow\{(\Omega,0)\}$ and $\mathcal L(\mathcal T)\leftarrow\varnothing$\;
\While{$\mathcal Q$ is not empty}{
  Pop a node $(B,d)$\;
  \eIf{$d=D$ or $B$ admits no valid split}{
    $\mathcal L(\mathcal T)\leftarrow\mathcal L(\mathcal T)\cup\{B\}$\;
  }{
    Construct $\operatorname{ch}(B)$ and compute $\Gamma_\tau(B)$ from Eq.~\eqref{eq:bmfa_refinement_score}\;
    \eIf{$\Gamma_\tau(B)>\varepsilon$}{
      Add $(C,d+1)$ to $\mathcal Q$ for every $C\in\operatorname{ch}(B)$\;
    }{
      $\mathcal L(\mathcal T)\leftarrow\mathcal L(\mathcal T)\cup\{B\}$\;
    }
  }
}
\end{algorithm}

Using the strict condition ``$>$'' makes the behavior at equality deterministic. Decreasing the threshold cannot reduce the number of leaves, and increasing the maximum depth expands the finest attainable resolution. A node that still has a large $\Gamma_\tau(B)$ at depth $D$ must be marked as depth-limited; the current method does not claim that such a node satisfies the exact error threshold.

\subsection{Piecewise-Constant Attention Reconstruction and Multiplicity Preservation}
\label{subsec:attention_reconstruction}

After tree construction, every score in a leaf is replaced by the leaf mean, yielding $q^{\mathcal T}$ in Eq.~\eqref{eq:tree_approximation}. The operation retains the original token count and isolates the effects of spatial selection and output fidelity. A future sparse implementation that compresses each leaf into a single representative token must preserve the number of original tokens contained in the leaf.

\begin{proposition}[Equivalence of repeated-mean and single-representative forms]
\label{prop:multiplicity_equivalence}
For each leaf $L$, define the representative logit and value as
\begin{equation}
 \ell_L=M(L)+\tau\log n_L,
 \qquad
 \overline{\bm v}_L=\frac{1}{n_L}\sum_{i\in L}\bm v_i.
\label{eq:representative_token}
\end{equation}
Computing a softmax-weighted sum over the representative pairs $(\ell_L,\overline{\bm v}_L)$ yields exactly the same attention output as repeating the mean score $M(L)$ at every original position $i\in L$.
\end{proposition}

\begin{proof}
In the repeated-mean representation, leaf $L$ contributes $n_Le^{M(L)/\tau}$ to the approximate partition function and contributes
$e^{M(L)/\tau}\sum_{i\in L}\bm v_i=n_Le^{M(L)/\tau}\overline{\bm v}_L$ to the attention numerator. The representative form satisfies $e^{\ell_L/\tau}=e^{M(L)/\tau}e^{\log n_L}=n_Le^{M(L)/\tau}$. The partition-function and numerator contributions therefore match exactly. Summation over all leaves produces identical normalized outputs.
\end{proof}

Omitting $\tau\log n_L$ assigns the same prior multiplicity to leaves of different area and violates Proposition~\ref{prop:multiplicity_equivalence}. The current closed-loop prototype replaces scores on the original grid, so repeated tokens preserve multiplicity automatically. A future single-representative kernel must add the term explicitly.

\subsection{Structural Properties of Threshold, Depth, and Tree Size}
\label{subsec:method_properties}

\begin{proposition}[Threshold nesting]
\label{prop:threshold_nesting}
For fixed input, temperature, and maximum depth, if $0\leq\varepsilon_1\leq\varepsilon_2$, the tree generated with $\varepsilon_1$ refines the tree generated with $\varepsilon_2$. Hence
$|\mathcal L(\mathcal T_{\varepsilon_1})|\geq|\mathcal L(\mathcal T_{\varepsilon_2})|$.
\end{proposition}

\begin{proof}[Proof sketch]
Induction on tree depth shows that every node split under the larger threshold also satisfies the split condition under the smaller threshold. Lowering the threshold can add branches but cannot remove an existing branch. The complete proof appears in Supplementary Section~\ref{app:proof_structure}.
\end{proof}

\begin{proposition}[Maximum-depth monotonicity]
\label{prop:depth_monotonicity}
For fixed input, temperature, and threshold, increasing the maximum depth from $D_1$ to $D_2\geq D_1$ cannot decrease the number of leaves or the proxy free energy $\widetilde F_\tau(\mathcal T)$ associated with the final leaf partition.
\end{proposition}

\begin{proof}[Proof sketch]
Increasing the maximum depth leaves every decision up to depth $D_1$ unchanged and can only subdivide leaves previously stopped by the depth limit. The leaf count cannot decrease, and Theorem~\ref{thm:partition_monotonicity} ensures that the proxy free energy does not decrease under refinement. Supplementary Section~\ref{app:proof_structure} provides the complete proof.
\end{proof}

The monotonicity properties apply to tree structure and the free-energy lower bound. They do not imply strictly monotone downstream Top-1 accuracy, because subsequent nonlinear layers and finite classification margins can produce non-monotone discrete predictions.

\begin{theorem}[Termination and worst-case tree size]
\label{thm:termination_complexity}
For a finite grid $\Omega$ and a finite maximum depth $D$, Algorithm~\ref{alg:bmfa} terminates after finitely many steps. If every internal node has exactly four children, then
$|\mathcal L(\mathcal T)|\leq\min(N,4^D)$, and the total number of visited nodes equals $(4|\mathcal L(\mathcal T)|-1)/3$.
\end{theorem}

\begin{proof}[Proof sketch]
A finite branching factor and finite depth imply a finite tree. The leaf-count bounds follow from the complete depth-$D$ quadtree and the single-token partition. For a full quadtree, the total-node formula follows from the edge identity $4I=I+L-1$. Supplementary Section~\ref{app:proof_structure} gives the complete counting argument.
\end{proof}

If two-dimensional prefix sums make rectangular mean queries $O(1)$, tree-construction time is linear in the number of visited nodes, and memory is of the same order as the queue and leaf set. For irregular boundary splits with two or three children, the exact total-node formula changes, while finite termination and $O(\text{number of visited nodes})$ complexity remain valid.

\subsection{Boundary Conditions, Numerical Implementation, and Computational Scope}
\label{subsec:boundary_implementation}

A homogeneous block satisfies $\Gamma_\tau(B)=G_\tau(B)=0$ and requires no refinement. A single-token block satisfies $F_\tau(B)=M(B)$ and has no valid children. Non-square grids can use unequal rectangular children, but Eq.~\eqref{eq:bmfa_refinement_score} must employ the actual valid-token weights $n_C/n_B$ rather than assigning equal weight to four children. Padding and invalid positions must be excluded from $\Omega$ and all block statistics; classification tokens and other non-spatial positions remain separate.

The partition proxy uses a numerically stable log-sum-exp implementation. Let $a_C=M(C)/\tau$ and $a_{\max}=\max_Ca_C$. Then
\begin{equation}
 \Phi_\tau(B;\mathcal P)
 =\tau\left[
 a_{\max}+
 \log\sum_{C\in\mathcal P(B)}
 \frac{n_C}{n_B}e^{a_C-a_{\max}}
 \right].
\label{eq:stable_partition_lse}
\end{equation}
Equation~\eqref{eq:stable_partition_lse} is mathematically identical to Eq.~\eqref{eq:partition_proxy}, while every exponential argument is non-positive.

The current prototype constructs the tree and replaces block scores after the full QK score matrix has been computed. The implementation isolates the refinement criterion, prediction agreement, and KL fidelity, but does not demonstrate end-to-end sparse acceleration. Reported leaf ratios represent potential granularity, and Python/NumPy runtime measures diagnostic overhead. Converting BMFA into realized speedup requires a reliable low-cost score proxy before the full matrix multiplication and a fused kernel capable of skipping the corresponding QK or downstream token computation. Proposition~\ref{prop:score_stability} specifies the error margin that such a proxy must satisfy.

\section{Experiments}
\label{sec:experiments}

\subsection{Experimental Protocol}
\label{subsec:experimental_setup}

The evaluation follows four levels: controlled mechanism diagnosis, real-data probes, closed-loop model fidelity, and standard image classification. The synthetic experiment varies minority area ratio $\alpha$ and logit advantage $\Delta$ on a $16\times16$ score grid and aggregates underestimation only when minority Gibbs mass is at least $0.5$, resulting in $396$ valid parameter settings per method. Real-data probes use all $5{,}000$ COCO val2017 images and $36{,}149$ instances, with images processed at $256\times256$ and mapped to a $16\times16$ grid. The image-edge probe restricts minority regions using instance annotations and modulates local scores with image-edge responses; the mask-only probe retains instance geometry without pixel-edge magnitude. The LVIS probe applies the same statistical protocol to $20{,}000$ instances from $1{,}999$ images. Annotations participate only in diagnostic score construction and are not inference inputs to BMFA. Complete data processing, model configuration, hardware, and reporting conventions are provided in Supplementary Section~\ref{app:detailed_experimental_protocol}.

Comparisons include the mean-based block summary denoted as Sparge-style mean, a second-order mean-variance correction, fixed-depth quadtrees, the adaptive BMFA quadtree, and Edge-aware keep, which preserves all annotated minority positions. Sparge-style mean is a mechanism-level proxy for mean-based block screening and is not the full official SpargeAttention kernel of \citet{zhang2025spargeattn}. Edge-aware keep uses diagnostic annotations and therefore serves only as a zero-error oracle rather than a deployable method. Closed-loop and ImageNet-1K experiments use pretrained DeiT-Tiny/16 from \citet{touvron2021deit}. Because the materials contain neither repeated random runs nor standard deviations, all figures report the recorded aggregate values without error bars or significance markers.

\subsection{Controlled Boundary-Minority Stress Test}
\label{subsec:synthetic_experiment}

Figure~\ref{fig:synthetic_failure_curve} fixes $\alpha=4/256$ and varies the minority logit advantage $\Delta$. Proposition~\ref{prop:minority_threshold} gives the critical point at which the minority state reaches one-half of the Gibbs mass: $\Delta=\log 63\approx4.143$. Beyond the threshold, underestimation from the mean summary, variance correction, and fixed partitions continues to increase with response advantage, whereas BMFA triggers local refinement and reduces the residual error of the controlled structure to near zero. The curve directly evaluates the mass-reversal regime identified by the theoretical analysis rather than a trend induced by natural-image noise or class imbalance.

\begin{figure*}[t]
  \centering
  \includegraphics[width=0.92\textwidth]{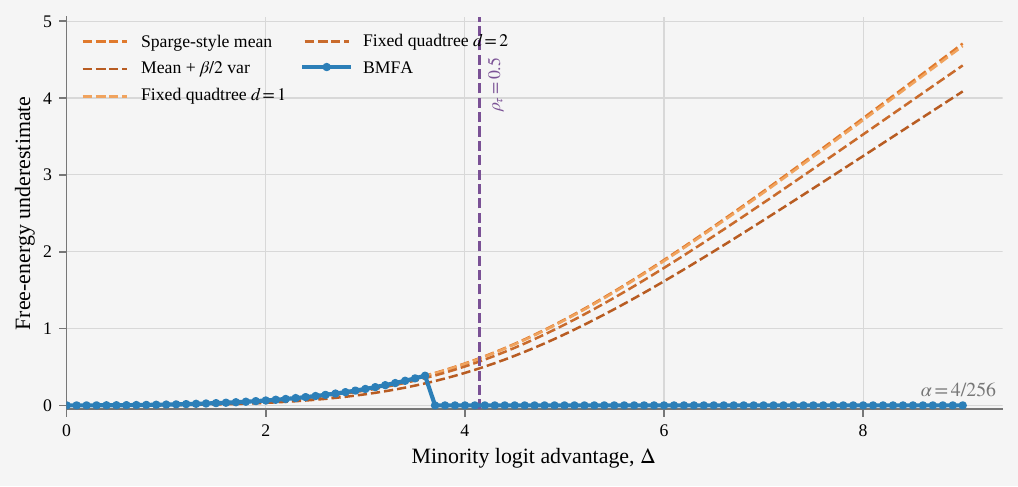}
  \caption{Controlled boundary-minority stress test. The minority area ratio is fixed at $\alpha=4/256$, and the horizontal axis denotes the logit advantage $\Delta$. The purple dashed line marks the theoretical threshold at which minority Gibbs mass reaches $0.5$. BMFA resolves the minority high-response structure after the lower-bound refinement score is activated, whereas the mean summary, second-order moment correction, and shallow fixed partitions retain an underestimate that grows with $\Delta$. Curves are plotted directly from the complete parameter-sweep records.}
  \label{fig:synthetic_failure_curve}
\end{figure*}

Across all $396$ minority-dominant settings, Sparge-style mean produces mean and P95 underestimates of $2.582$ and $5.114$. BMFA reduces the two values to $0.261$ and $2.035$ with an average leaf ratio of $5.794\%$. A fixed quadtree with $d=2$ uses a comparable $6.250\%$ leaf ratio but retains a mean underestimate of $1.129$, showing that the improvement does not follow automatically from a modest increase in granularity and depends on refinement location. Edge-aware keep reaches zero error at a $100\%$ leaf ratio and therefore defines only a full-retention upper bound.

\subsection{Real-Data Mechanism Validation and Cross-Dataset Generalization}
\label{subsec:real_probe}

Figure~\ref{fig:probe_pareto} compares synthetic, COCO image-edge, and LVIS image-edge probes in a common coordinate system. BMFA remains in the low-mean-error region for all three data sources and substantially outperforms fixed partitions of comparable leaf ratio. The trend indicates that mean underestimation in the binary model is not restricted to an artificial shape distribution. On COCO, BMFA reduces the mean underestimate of Sparge-style mean from $2.254$ to $0.526$, a relative reduction of $76.7\%$, at a $30.810\%$ average leaf ratio. On LVIS, the corresponding reduction is from $2.277$ to $0.596$, or $73.8\%$, at a $29.225\%$ leaf ratio.

\begin{figure*}[t]
  \centering
  \includegraphics[width=0.98\textwidth]{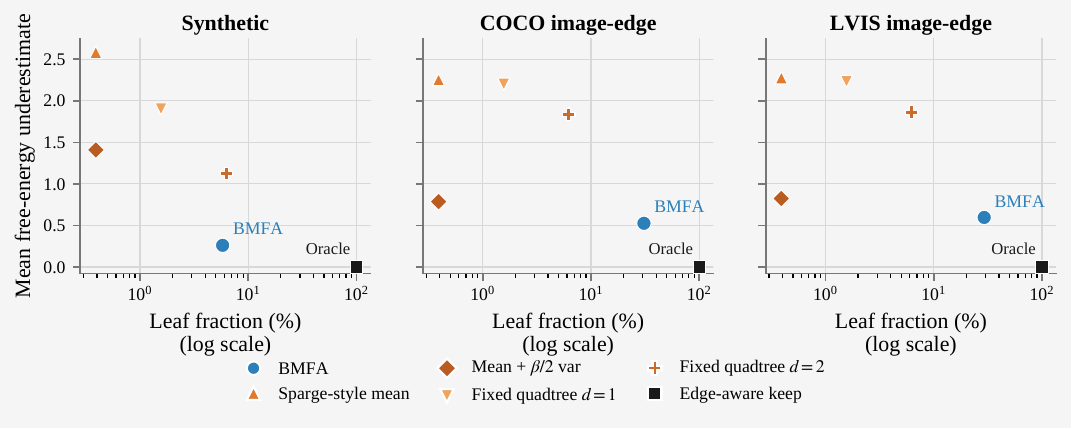}
  \caption{Trade-off between mean free-energy underestimate and leaf ratio on synthetic, COCO, and LVIS probes. The horizontal axis uses a logarithmic scale to display both the $0.391\%$ single-block summary and the $100\%$ oracle. BMFA attains lower mean error than shallow fixed partitions across all three data sources. Edge-aware keep is a full-retention upper bound that uses diagnostic boundaries.}
  \label{fig:probe_pareto}
\end{figure*}

\begin{table*}[t]
\centering
\caption{Complete image-edge probe results on COCO and LVIS. Lower gap is better; Leaf denotes the average leaf ratio. Every method is evaluated on the same instance set. The best deployable result is shown in bold, while the annotation-dependent oracle is excluded from best-value comparisons.}
\label{tab:real_probe}
\resizebox{0.96\textwidth}{!}{%
\begin{tabular}{lrrrrrr}
\toprule
& \multicolumn{3}{c}{COCO val2017, $36{,}149$ instances} & \multicolumn{3}{c}{LVIS, $20{,}000$ instances}\\
\cmidrule(lr){2-4}\cmidrule(lr){5-7}
Method & Mean gap & P95 gap & Leaf (\%) & Mean gap & P95 gap & Leaf (\%)\\
\midrule
Sparge-style mean & 2.254 & 2.551 & 0.391 & 2.277 & 2.613 & 0.391\\
Mean + $\beta/2$ var & 0.787 & \textbf{1.162} & 0.391 & 0.824 & \textbf{1.248} & 0.391\\
Fixed quadtree $d=1$ & 2.200 & 2.524 & 1.563 & 2.233 & 2.590 & 1.563\\
Fixed quadtree $d=2$ & 1.840 & 2.241 & 6.250 & 1.865 & 2.313 & 6.250\\
\textbf{BMFA} & \textbf{0.526} & 2.434 & 30.810 & \textbf{0.596} & 2.509 & 29.225\\
Edge-aware keep (oracle) & 0.000 & 0.000 & 100.000 & 0.000 & 0.000 & 100.000\\
\bottomrule
\end{tabular}}
\end{table*}

Table~\ref{tab:real_probe} also reveals the current limitation of BMFA. BMFA yields the lowest mean gap, but its COCO and LVIS P95 values are higher than those of the second-order variance correction. The behavior agrees with Proposition~\ref{prop:one_step_counterexample}: a one-step score identifies substantial heterogeneity across child blocks but can miss deeper structures whose means cancel at the current scale. The mask-only probe provides an independent check. BMFA obtains a mean gap of $0.534$, below the $0.600$ value of Mean+$\beta/2$var, while the corresponding P95 values are $1.830$ and $0.835$. The evidence therefore supports reduced average boundary-minority underestimation and adaptive localization of critical regions, but does not support uniform dominance over the second-order approximation at every quantile.

\subsection{Threshold Ablation}
\label{subsec:ablation}

Figure~\ref{fig:epsilon_ablation} reports the effects of $\varepsilon$ on mean error, leaf ratio, and runtime of the unfused diagnostic implementation. Decreasing the threshold from $0.05$ to $0.001$ reduces mean gap from $1.644$ to $0.120$, increases average leaf ratio from $9.26\%$ to $53.24\%$, and raises runtime from $0.697$ to $3.350$ ms per instance. The default configuration $\varepsilon=0.005$, $D=4$, and grid size $16$ obtains a mean gap of $0.549$ at a $29.08\%$ leaf ratio with $2.099$ ms per-instance Python/NumPy overhead.

The complete threshold curves and runtime visualization are reported in Supplementary Figure~\ref{fig:epsilon_ablation}.

Increasing maximum depth from $2$ to $4$ reduces mean gap from $1.917$ to $0.549$; a further increase to $5$ changes neither the gap nor leaf ratio, indicating that depth $4$ reaches the finest level resolved by the criterion on a $16\times16$ grid. Grid size $8$ reduces mean gap to $0.403$ but increases leaf ratio to $56.03\%$. Grid size $32$ reduces leaf ratio to $12.39\%$ but increases mean gap to $0.910$ and runtime to $3.553$ ms per instance. Grid size $16$ therefore provides the current compromise among approximation error, granularity, and implementation overhead. Complete depth and grid tables are reported in Supplementary Section~\ref{app:depth_grid_ablation}.

\subsection{Closed-Loop Attention-Output Fidelity}
\label{subsec:closed_loop}

Mechanism probes measure free-energy approximation but do not alone establish preservation of model outputs. Figure~\ref{fig:closed_loop_fidelity} inserts the same screening rule into the last four DeiT-Tiny blocks and compares outputs with the full model on $5{,}000$ COCO images. BMFA with $\varepsilon=0.8$ uses a $30.82\%$ average leaf ratio, reaches $75.98\%$ Top-1 agreement, and yields $D_{\mathrm{KL}}=0.149$. Random retention with $p=0.25$ uses a $26.53\%$ leaf ratio but reaches only $61.54\%$ agreement and $D_{\mathrm{KL}}=0.406$. In the higher-budget regime, BMFA with $\varepsilon=0.5$ and random retention with $p=0.5$ use leaf ratios of $52.66\%$ and $51.00\%$, respectively; their agreements are $87.00\%$ and $72.40\%$, and their KL divergences are $0.041$ and $0.189$. The persistent differences at comparable budgets show that output fidelity arises primarily from free-energy-based spatial selection rather than retention quantity alone.

\begin{figure*}[t]
  \centering
  \includegraphics[width=0.94\textwidth]{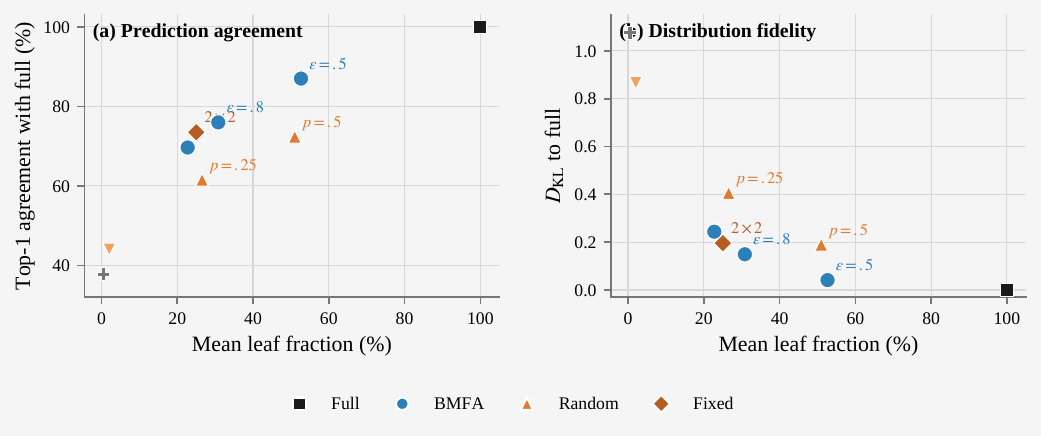}
  \caption{Closed-loop fidelity on COCO 5k when replacing attention logits in the last four DeiT-Tiny blocks. (a) Top-1 prediction agreement with the full model. (b) KL divergence from the approximate output distribution to the full-model distribution. At leaf budgets comparable to random retention, BMFA yields higher agreement and lower KL. Global mean and very shallow fixed partitions occupy the low-budget, low-fidelity region.}
  \label{fig:closed_loop_fidelity}
\end{figure*}

A fixed $2\times2$ partition uses a $25.00\%$ leaf ratio and obtains $73.50\%$ agreement with $0.196$ KL. BMFA with $\varepsilon=0.8$ uses a moderately larger $30.82\%$ ratio but improves agreement by $2.48$ percentage points and lowers KL. The variance-only configuration requires a $97.57\%$ leaf ratio to reach $99.02\%$ agreement and therefore does not demonstrate effective low-granularity selection. Approximating more layers degrades every method, but the advantage of BMFA over budget-matched random retention persists for last-2, last-4, last-8, and all-layer settings, as shown in Supplementary Figure~\ref{fig:layer_sensitivity}.

\subsection{ImageNet-1K Model-Level Evaluation}
\label{subsec:imagenet}

Figure~\ref{fig:imagenet_tradeoffs}(a) compares BMFA with controls that have explicit leaf ratios. BMFA with $\varepsilon=0.5$ reaches $71.520\%$ Top-1 accuracy at a $55.861\%$ average leaf ratio, a decrease of $0.564$ percentage points from the full DeiT-Tiny result of $72.084\%$, with $93.646\%$ Top-1 agreement. Random retention with $p=0.5$ uses a comparable $51.024\%$ leaf ratio but reaches only $69.148\%$ Top-1. BMFA with $\varepsilon=0.8$ and random retention with $p=0.25$ attain $69.820\%$ and $65.146\%$ Top-1, respectively. A fixed $2\times2$ partition reaches $69.322\%$ at a $25\%$ leaf ratio, below the $69.820\%$ result of BMFA with $\varepsilon=0.8$. The comparisons again indicate that adaptive spatial allocation is more effective than fixed or random assignment.

\begin{figure*}[t]
  \centering
  \includegraphics[width=0.97\textwidth]{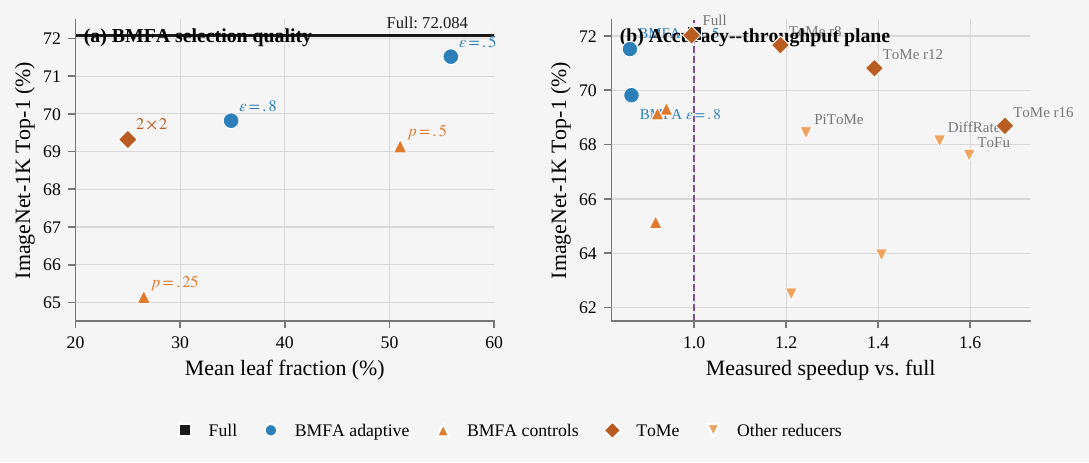}
  \caption{Model-level results on the ImageNet-1K validation set. (a) Leaf-ratio--Top-1 trade-off for adaptive BMFA, random retention, and fixed controls; the black horizontal line denotes full DeiT-Tiny. (b) Measured accuracy--throughput plane for all local implementations; the purple dashed line denotes full-model throughput. Because the current BMFA prototype performs score replacement after full QK computation, panel (b) reports prototype overhead rather than a speed Pareto claim.}
  \label{fig:imagenet_tradeoffs}
\end{figure*}

\begin{table*}[t]
\centering
\caption{Representative ImageNet-1K validation results. Speedup is computed from measured images per second relative to the full model in the same environment. Current BMFA results evaluate selection quality and do not represent an optimized sparse kernel.}
\label{tab:imagenet_results}
\resizebox{0.90\textwidth}{!}{%
\begin{tabular}{lrrrrr}
\toprule
Method & Top-1 (\%) & Drop (pp) & Agreement (\%) & Leaf (\%) & Speedup\\
\midrule
Full DeiT-Tiny & 72.084 & 0.000 & -- & -- & 1.000$\times$\\
ToMe $r=8$~\citep{bolya2023tome} & 71.666 & 0.418 & 94.122 & -- & 1.187$\times$\\
Fixed $2\times2$ & 69.322 & 2.762 & 87.322 & 25.000 & 0.940$\times$\\
BMFA $\varepsilon=0.8$ & 69.820 & 2.264 & 88.146 & 34.868 & 0.864$\times$\\
\textbf{BMFA $\varepsilon=0.5$} & \textbf{71.520} & \textbf{0.564} & \textbf{93.646} & 55.861 & 0.860$\times$\\
Random $p=0.25$ & 65.146 & 6.938 & 79.458 & 26.530 & 0.916$\times$\\
Random $p=0.5$ & 69.148 & 2.936 & 86.908 & 51.024 & 0.920$\times$\\
\bottomrule
\end{tabular}}
\end{table*}

Figure~\ref{fig:imagenet_tradeoffs}(b) additionally includes ToMe, PiToMe \citep{tran2024pitome}, DiffRate \citep{chen2023diffrate}, CrossGET \citep{shi2024crossget}, ToFu \citep{kim2024tofu}, and a local DCT token-reduction implementation. ToMe with $r=8$ reaches $71.666\%$ Top-1 and $1.187\times$ measured speedup, providing a stronger accuracy-throughput point in the current table. BMFA throughput is below that of the full model, so the present contribution is restricted to boundary-minority failure analysis, selection quality, and output fidelity. Transforming the leaf structure into speedup requires predicting the refinement score before full QK computation and implementing a fused sparse or token-reduction kernel.

\subsection{Qualitative Analysis on Real Instances}
\label{subsec:qualitative}

Figure~\ref{fig:qualitative_cases} presents four COCO instances from the kite, person, book, and spoon categories. Each example contains the image crop, the boundary-minority response field constructed from the instance boundary, and the adaptive BMFA refinement. Minority area ratios range from $9.8\%$ to $12.5\%$, Gibbs masses exceed $97\%$, and free-energy gaps range from $2.854$ to $3.091$. Fine leaves concentrate around contours, elongated structures, and small parts, while extensive near-homogeneous regions retain coarse representations. The spatial allocation agrees with the theoretical local-heterogeneity criterion and does not imply that an external boundary label is required during inference.

Supplementary Figure~\ref{fig:qualitative_cases} provides the complete visualization panels and per-instance statistics.

\subsection{Scope of the Empirical Claims and Current Limitations}
\label{subsec:experimental_limitations}

The experiments support three claims. Mean-based coarse summaries systematically underestimate free energy in boundary-minority regions; BMFA substantially reduces average underestimation across synthetic, COCO, and LVIS data; and adaptive selection preserves DeiT outputs and ImageNet accuracy more effectively than random or fixed controls at comparable leaf budgets. The evidence also has explicit boundaries. First, the one-step score is a lower-bound signal, and the P95 tail does not uniformly outperform variance correction. Second, the real-data probes use annotation-informed score construction to validate the mechanism rather than to emulate annotation-free deployment. Third, the prototype applies the approximation after full QK computation and has not demonstrated sparse acceleration. Future evaluation should examine multi-level look-ahead or dual-bound certification and combine a low-cost score proxy with a hardware-efficient kernel.

\section{Conclusion}
\label{sec:conclusion}
We studied the failure of coarse token and attention summaries in boundary-minority regions, where spatially rare high responses can dominate exponential attention mass while contributing negligibly to a block mean. The theoretical analysis connects the free-energy gap to KL divergence, establishes minority-mass reversal and asymptotic mean blindness, and clarifies why finite-order moment corrections cannot uniformly resolve the limiting regime. BMFA operationalizes the analysis through a hierarchical lower-bound refinement score and an adaptive quadtree that allocates resolution to unresolved heterogeneous blocks. Synthetic, COCO, LVIS, closed-loop DeiT-Tiny, and ImageNet-1K evaluations consistently support the selection-quality claim. The one-step score remains a lower-bound signal rather than a complete certificate, and the present implementation operates after full QK computation. Multi-level look-ahead, certified upper bounds, low-cost score prediction, and fused sparse or token-reduction kernels constitute the principal directions required to convert the demonstrated selection fidelity into end-to-end acceleration.

\begingroup
\footnotesize
\bibliographystyle{unsrtnat}
\bibliography{references}

@inproceedings{dosovitskiy2021vit,
  title     = {An Image Is Worth 16x16 Words: Transformers for Image Recognition at Scale},
  author    = {Dosovitskiy, Alexey and Beyer, Lucas and Kolesnikov, Alexander and Weissenborn, Dirk and Zhai, Xiaohua and Unterthiner, Thomas and Dehghani, Mostafa and Minderer, Matthias and Heigold, Georg and Gelly, Sylvain and Uszkoreit, Jakob and Houlsby, Neil},
  booktitle = {International Conference on Learning Representations},
  year      = {2021},
}

@inproceedings{liu2021swin,
  title     = {Swin Transformer: Hierarchical Vision Transformer Using Shifted Windows},
  author    = {Liu, Ze and Lin, Yutong and Cao, Yue and Hu, Han and Wei, Yixuan and Zhang, Zheng and Lin, Stephen and Guo, Baining},
  booktitle = {Proceedings of the IEEE/CVF International Conference on Computer Vision},
  pages     = {10012--10022},
  year      = {2021},
}

@inproceedings{strudel2021segmenter,
  title     = {Segmenter: Transformer for Semantic Segmentation},
  author    = {Strudel, Robin and Garcia, Ricardo and Laptev, Ivan and Schmid, Cordelia},
  booktitle = {Proceedings of the IEEE/CVF International Conference on Computer Vision},
  pages     = {7262--7272},
  year      = {2021},
}

@inproceedings{radford2021clip,
  title     = {Learning Transferable Visual Models From Natural Language Supervision},
  author    = {Radford, Alec and Kim, Jong Wook and Hallacy, Chris and Ramesh, Aditya and Goh, Gabriel and Agarwal, Sandhini and Sastry, Girish and Askell, Amanda and Mishkin, Pamela and Clark, Jack and Krueger, Gretchen and Sutskever, Ilya},
  booktitle = {Proceedings of the 38th International Conference on Machine Learning},
  pages     = {8748--8763},
  year      = {2021},
  volume    = {139},
  series    = {Proceedings of Machine Learning Research},
}

@inproceedings{touvron2021deit,
  title     = {Training Data-Efficient Image Transformers and Distillation Through Attention},
  author    = {Touvron, Hugo and Cord, Matthieu and Douze, Matthijs and Massa, Francisco and Sablayrolles, Alexandre and J{\'e}gou, Herv{\'e}},
  booktitle = {Proceedings of the 38th International Conference on Machine Learning},
  pages     = {10347--10357},
  year      = {2021},
  volume    = {139},
  series    = {Proceedings of Machine Learning Research},
}

@inproceedings{rao2021dynamicvit,
  title     = {DynamicViT: Efficient Vision Transformers with Dynamic Token Sparsification},
  author    = {Rao, Yongming and Zhao, Wenliang and Liu, Benlin and Lu, Jiwen and Zhou, Jie and Hsieh, Cho-Jui},
  booktitle = {Advances in Neural Information Processing Systems},
  volume    = {34},
  pages     = {13937--13949},
  year      = {2021},
}

@inproceedings{liang2022evit,
  title     = {Not All Patches Are What You Need: Expediting Vision Transformers via Token Reorganizations},
  author    = {Liang, Youwei and Ge, Chongjian and Tong, Zhan and Song, Yibing and Wang, Jue and Xie, Pengtao},
  booktitle = {International Conference on Learning Representations},
  year      = {2022},
}

@inproceedings{yin2022avit,
  title     = {A-ViT: Adaptive Tokens for Efficient Vision Transformer},
  author    = {Yin, Hongxu and Vahdat, Arash and Alvarez, Jose M. and Mallya, Arun and Kautz, Jan and Molchanov, Pavlo},
  booktitle = {Proceedings of the IEEE/CVF Conference on Computer Vision and Pattern Recognition},
  pages     = {10809--10818},
  year      = {2022},
}

@inproceedings{ryoo2021tokenlearner,
  title     = {TokenLearner: What Can 8 Learned Tokens Do for Images and Videos?},
  author    = {Ryoo, Michael S. and Piergiovanni, A. J. and Arnab, Anurag and Dehghani, Mostafa and Angelova, Anelia},
  booktitle = {Advances in Neural Information Processing Systems},
  volume    = {34},
  pages     = {23696--23708},
  year      = {2021},
}

@inproceedings{bolya2023tome,
  title     = {Token Merging: Your ViT but Faster},
  author    = {Bolya, Daniel and Fu, Cheng-Yang and Dai, Xiaoliang and Zhang, Peizhao and Feichtenhofer, Christoph and Hoffman, Judy},
  booktitle = {International Conference on Learning Representations},
  year      = {2023},
}

@inproceedings{chang2023stvit,
  title     = {Making Vision Transformers Efficient From a Token Sparsification View},
  author    = {Chang, Shuning and Wang, Pichao and Wang, Fan and Li, Hao and Feng, Jiashi},
  booktitle = {Proceedings of the IEEE/CVF Conference on Computer Vision and Pattern Recognition},
  pages     = {6195--6205},
  year      = {2023},
}

@inproceedings{zhang2025spargeattn,
  title     = {SpargeAttention: Accurate and Training-free Sparse Attention Accelerating Any Model Inference},
  author    = {Zhang, Jintao and Xiang, Chendong and Huang, Haofeng and Wei, Jia and Xi, Haocheng and Zhu, Jun and Chen, Jianfei},
  booktitle = {Proceedings of the 42nd International Conference on Machine Learning},
  pages     = {76397--76413},
  year      = {2025},
  volume    = {267},
  series    = {Proceedings of Machine Learning Research},
}

@inproceedings{lin2014coco,
  title     = {Microsoft COCO: Common Objects in Context},
  author    = {Lin, Tsung-Yi and Maire, Michael and Belongie, Serge and Bourdev, Lubomir and Girshick, Ross and Hays, James and Perona, Pietro and Ramanan, Deva and Doll{\'a}r, Piotr and Zitnick, C. Lawrence},
  booktitle = {European Conference on Computer Vision},
  pages     = {740--755},
  year      = {2014},
}

@inproceedings{gupta2019lvis,
  title     = {LVIS: A Dataset for Large Vocabulary Instance Segmentation},
  author    = {Gupta, Agrim and Doll{\'a}r, Piotr and Girshick, Ross},
  booktitle = {Proceedings of the IEEE/CVF Conference on Computer Vision and Pattern Recognition},
  year      = {2019},
}

@inproceedings{deng2009imagenet,
  title     = {ImageNet: A Large-Scale Hierarchical Image Database},
  author    = {Deng, Jia and Dong, Wei and Socher, Richard and Li, Li-Jia and Li, Kai and Fei-Fei, Li},
  booktitle = {2009 IEEE Conference on Computer Vision and Pattern Recognition},
  pages     = {248--255},
  year      = {2009},
}

@inproceedings{wang2024zerotprune,
  title     = {Zero-TPrune: Zero-Shot Token Pruning through Leveraging of the Attention Graph in Pre-Trained Transformers},
  author    = {Wang, Hongjie and Dedhia, Bhishma and Jha, Niraj K.},
  booktitle = {Proceedings of the IEEE/CVF Conference on Computer Vision and Pattern Recognition},
  pages     = {16070--16079},
  year      = {2024},
}

@inproceedings{kim2024tofu,
  title     = {Token Fusion: Bridging the Gap Between Token Pruning and Token Merging},
  author    = {Kim, Minchul and Gao, Shangqian and Hsu, Yen-Chang and Shen, Yilin and Jin, Hongxia},
  booktitle = {Proceedings of the IEEE/CVF Winter Conference on Applications of Computer Vision},
  pages     = {1383--1392},
  year      = {2024},
}

@inproceedings{lee2024mctf,
  title     = {Multi-criteria Token Fusion with One-step-ahead Attention for Efficient Vision Transformers},
  author    = {Lee, Sanghyeok and Choi, Joonmyung and Kim, Hyunwoo J.},
  booktitle = {Proceedings of the IEEE/CVF Conference on Computer Vision and Pattern Recognition},
  pages     = {15741--15750},
  year      = {2024},
}

@inproceedings{norouzi2024algm,
  title     = {ALGM: Adaptive Local-then-Global Token Merging for Efficient Semantic Segmentation with Plain Vision Transformers},
  author    = {Norouzi, Narges and Orlova, Svetlana and de Geus, Daan and Dubbelman, Gijs},
  booktitle = {Proceedings of the IEEE/CVF Conference on Computer Vision and Pattern Recognition},
  pages     = {15773--15782},
  year      = {2024},
}

@inproceedings{tran2024pitome,
  title     = {Accelerating Transformers with Spectrum-Preserving Token Merging},
  author    = {Tran, Hoai-Chau and Nguyen, Duy M. H. and Nguyen, Duy M. and Nguyen, Trung-Tin and Le, Ngan and Xie, Pengtao and Sonntag, Daniel and Zou, James and Nguyen, Binh T. and Niepert, Mathias},
  booktitle = {Advances in Neural Information Processing Systems},
  volume    = {37},
  year      = {2024},
}

@inproceedings{xie2026sadtm,
  title     = {Saliency-Driven Token Merging for Vision Transformers},
  author    = {Xie, Weiying and Chen, Xiaoyu and Zhang, Xin and Hao, Chenhe and Ma, Jitao and Li, Yunsong and Fang, Leyuan},
  booktitle = {Proceedings of the IEEE/CVF Conference on Computer Vision and Pattern Recognition},
  pages     = {32184--32193},
  year      = {2026},
}

@inproceedings{ribar2024sparq,
  title     = {SparQ Attention: Bandwidth-Efficient LLM Inference},
  author    = {Ribar, Luka and Chelombiev, Ivan and Hudlass-Galley, Luke and Blake, Charlie and Luschi, Carlo and Orr, Douglas},
  booktitle = {Proceedings of the 41st International Conference on Machine Learning},
  pages     = {42558--42583},
  year      = {2024},
  volume    = {235},
  series    = {Proceedings of Machine Learning Research},
}

@inproceedings{tang2024quest,
  title     = {QUEST: Query-Aware Sparsity for Efficient Long-Context LLM Inference},
  author    = {Tang, Jiaming and Zhao, Yilong and Zhu, Kan and Xiao, Guangxuan and Kasikci, Baris and Han, Song},
  booktitle = {Proceedings of the 41st International Conference on Machine Learning},
  pages     = {47901--47911},
  year      = {2024},
  volume    = {235},
  series    = {Proceedings of Machine Learning Research},
}

@inproceedings{yuan2025nsa,
  title     = {Native Sparse Attention: Hardware-Aligned and Natively Trainable Sparse Attention},
  author    = {Yuan, Jingyang and Gao, Huazuo and Dai, Damai and Luo, Junyu and Zhao, Liang and Zhang, Zhengyan and Xie, Zhenda and Wei, Yuxing and Wang, Lean and Xiao, Zhiping and Wang, Yuqing and Ruan, Chong and Zhang, Ming and Liang, Wenfeng and Zeng, Wangding},
  booktitle = {Proceedings of the 63rd Annual Meeting of the Association for Computational Linguistics (Volume 1: Long Papers)},
  pages     = {23078--23097},
  year      = {2025},
  publisher = {Association for Computational Linguistics},
  doi       = {10.18653/v1/2025.acl-long.1126},
}

@inproceedings{vaswani2017attention,
  title     = {Attention Is All You Need},
  author    = {Vaswani, Ashish and Shazeer, Noam and Parmar, Niki and Uszkoreit, Jakob and Jones, Llion and Gomez, Aidan N. and Kaiser, Lukasz and Polosukhin, Illia},
  booktitle = {Advances in Neural Information Processing Systems},
  volume    = {30},
  year      = {2017},
}

@book{boyd2004convex,
  title     = {Convex Optimization},
  author    = {Boyd, Stephen and Vandenberghe, Lieven},
  publisher = {Cambridge University Press},
  year      = {2004},
}

@article{finkel1974quadtree,
  title   = {Quad Trees: A Data Structure for Retrieval on Composite Keys},
  author  = {Finkel, Raphael A. and Bentley, Jon Louis},
  journal = {Acta Informatica},
  volume  = {4},
  number  = {1},
  pages   = {1--9},
  year    = {1974},
  doi     = {10.1007/BF00288933},
}

@article{hoeffding1963probability,
  title   = {Probability Inequalities for Sums of Bounded Random Variables},
  author  = {Hoeffding, Wassily},
  journal = {Journal of the American Statistical Association},
  volume  = {58},
  number  = {301},
  pages   = {13--30},
  year    = {1963},
  doi     = {10.1080/01621459.1963.10500830},
}

@article{sason2015reversepinsker,
  title   = {On Reverse Pinsker Inequalities},
  author  = {Sason, Igal},
  journal = {arXiv preprint arXiv:1503.07118},
  year    = {2015},
}

@inproceedings{chen2023diffrate,
  title={DiffRate: Differentiable Compression Rate for Efficient Vision Transformers},
  author={Chen, Mengzhao and Shao, Wenqi and Xu, Peng and Lin, Mingbao and Zhang, Kaipeng and Chao, Fei and Ji, Rongrong and Qiao, Yu and Luo, Ping},
  booktitle={Proceedings of the IEEE/CVF International Conference on Computer Vision},
  year={2023},
}

@inproceedings{shi2024crossget,
  title={CrossGET: Cross-Guided Ensemble of Tokens for Accelerating Vision-Language Transformers},
  author={Shi, Dachuan and Tao, Chaofan and Rao, Anyi and Yang, Zhendong and Yuan, Chun and Wang, Jiaqi},
  booktitle={Proceedings of the 41st International Conference on Machine Learning},
  year={2024},
}
\endgroup

\end{document}